\documentclass[letterpaper,10pt]{article}
\usepackage{style}
\newgeometry{vmargin={25mm}, hmargin={17mm,17mm}}
\usepackage[affil-it]{authblk}
\usepackage[numbers]{natbib}

\usepackage{microtype}
\usepackage{graphicx}
\usepackage{subcaption}
\usepackage{booktabs} 
\usepackage{multirow}
\usepackage[parfill]{parskip}
\definecolor{mydarkblue}{rgb}{0,0.08,0.60}
\hypersetup{ %
colorlinks=true,
linkcolor=mydarkblue,
citecolor=mydarkblue,
filecolor=mydarkblue,
urlcolor=mydarkblue,
}

\title{Goldilocks RL: Tuning Task Difficulty to Escape Sparse Rewards for Reasoning}
\makeatletter
\renewcommand{\@fnsymbol}[1]{%
   \ifcase#1\or $\dagger$\or *\fi}
\makeatother
\author[1]{Ilia Mahrooghi\thanks{Corresponding Author, ilia.mahrooghi@epfl.ch}}
\author[2]{Aryo Lotfi}
\author[1,2]{Emmanuel Abbe}

\affil[1]{EPFL}
\affil[2]{Apple}

\date{}
\begin{document}
\maketitle

\begin{abstract} \label{sec:abs}
Reinforcement learning has emerged as a powerful paradigm for unlocking reasoning capabilities in language models. However, relying on sparse rewards makes this process highly sample-inefficient, as models must navigate vast search spaces with minimal feedback. While classic curriculum learning aims to mitigate this by ordering data based on complexity, prior works have primarily targeted small datasets and do not directly transfer to the large-scale settings typical of modern language model training. Furthermore, the right ordering for a specific model is often unclear. 
To address this, we propose Goldilocks, an adaptive data-selection strategy that uses a Selector network to predict the standard
deviation of rewards across the model's rollouts for each candidate question. 
The Selector prioritizes questions with high predicted reward variability, corresponding to questions that are neither too easy nor too hard for the model’s current capabilities (Goldilocks principle), while training the model with GRPO. By leveraging the model's performance on seen samples, the Selector continuously adapts to the model's evolving abilities. Across the OpenMathReasoning and Polaris datasets, Goldilocks consistently improves over standard GRPO, requiring up to 78\% fewer optimization steps to reach the corresponding GRPO performance.
\end{abstract}

\section{Introduction} \label{sec:intro}
The reasoning capabilities of Transformer-based models have been a focal point of recent research, with Reinforcement Learning (RL) emerging as a powerful paradigm for enhancing these abilities. Recent RL methods incentivize the generation of a detailed Chain of Thought (CoT) \citep{wei2022chain}, enabling models to solve complex problems \citep{zelikman2022star, deepseek2025incentivizing, hou2025t1}. Additionally, increasing the computational budget at inference time, known as scaling test-time compute, has been shown to further enhance these reasoning capabilities \citep{snell2024scaling}. While RL has proven effective in unlocking complex reasoning behaviors, the computational efficiency of these methods remains a significant challenge. Recent findings by \citet{khatri2025art} underscore that scaling is a fundamental driver of performance. This observation suggests that while models have the potential for continued improvement, realizing this potential through brute-force scaling is resource-intensive, making training efficiency a paramount concern.

Reinforcement learning with verifiable rewards (RLVR)
\citep{cobbe2021training,shao2024deepseekmath} provides a scalable way to
improve reasoning by assigning rewards using automatically checkable outcomes.
However, when the reward is determined primarily by final-answer correctness,
the resulting feedback is sparse, since a generated trajectory is primarily
judged by whether it reaches a correct final answer: a generated trajectory is rewarded if
it reaches a correct answer. This creates a substantial efficiency challenge during
training. Questions that are far beyond the model's current capabilities are
likely to produce groups of uniformly incorrect rollouts, while questions that
are already mastered tend to produce uniformly correct ones. In either case, substantial generation compute may be spent on rollout groups that provide little or no within-group reward variation, and therefore little useful relative signal for the GRPO update.

This inefficiency becomes increasingly important as language-model training
moves toward larger data and compute regimes. Scaling studies have established
the importance of jointly increasing model capacity, data, and compute
\citep{hoffmann2022training}, while recent frontier models are trained on
trillions of tokens \citep{deepseek2024v3}. Reasoning post-training is following
a similar trend, with recent systems relying on increasingly large-scale
reinforcement learning
\citep{guo2025deepseek,yu2025dapo,khatri2025art}. Reasoning datasets are also
growing to millions of questions \citep{moshkov2025aimo2}. In such a large-pool regime, an effective
data-selection strategy cannot assume that the model has already interacted
with every candidate question. Instead, it should be able to estimate the
learning value of previously unseen questions from limited model feedback,
and ideally make this decision \emph{before} incurring the cost of generating
rollouts.

Curriculum Learning (CL) \citep{bengio2009curriculum} provides a natural
framework for prioritizing training examples, but many existing approaches rely
on information that becomes difficult to obtain in this regime. History-based
methods estimate the difficulty or utility of a question from the model's
previous performance on that same instance
\citep{qu2025can,shen2025bots}; such estimates become informative only after
sufficient observations of the question have been collected. Other approaches
schedule predefined difficulty categories or data partitions
\citep{parashar2025curriculum,chen2025self}, requiring structure that may not be
available for a large heterogeneous question pool. A different line of work
identifies uninformative examples only after generating their rollouts and then
skips the corresponding policy update
\citep{wang2025dump,xiong2025reinforce}; this avoids an unhelpful optimization
step, but not the generation cost already incurred. This motivates a curriculum that can infer the utility of previously
unseen questions from recent model feedback, make the selection before rollout
generation, and adapt as the model improves.

In this work, we propose \textbf{Goldilocks}, a pre-rollout data-selection framework that uses a learned Selector. Alongside model RL training, a lightweight
Selector network learns from recent model outcomes to predict the reward variability
of candidate questions. It then prioritizes questions with high reward variability that are expected to produce
informative rollout groups (neither those already mastered by the model nor
those currently beyond its capabilities). Because the prediction is made from
the question itself, the Selector can generalize to candidates the model has
never encountered and continuously update its curriculum as the model
improves. We show theoretically that this criterion is aligned with the learning signal
of GRPO \citep{shao2024deepseekmath} under binary verifiable rewards, and empirically evaluate it across
large- and small-pool training regimes against predictive and history-based
curriculum baselines.

The paper is organized as follows. Section~\ref{sec:related} discusses related work. Section~\ref{sec:grpo_verifiable_reward} analyzes GRPO gradient dynamics under verifiable rewards. Section~\ref{sec:training_proc} presents our joint training framework. Sections~\ref{sec:experiments} and~\ref{sec:ablation} report empirical results and ablations and Section~\ref{sec:conclusion} concludes.

\section{Related Work} \label{sec:related}

\textbf{Reasoning and Reinforcement Learning.}
Intermediate reasoning traces have been shown to substantially improve language-model performance on complex tasks, through supervised scratchpads \citep{nye2021scratchpads}, Chain-of-Thought prompting \citep{wei2022chain,kojima2022large}, and reinforcement learning \citep{shao2024deepseekmath}. Recent RL methods for reasoning optimize such trajectories using verifiable rewards, but their efficiency is limited by the large generation cost and by questions that provide little useful learning signal. This has motivated methods that improve which problems are selected for training or how rollout compute is allocated.

\textbf{History-Based Curricula and Adaptive Sampling.}
A large class of curriculum methods estimates question utility from previous interactions with the same question. MoPPS \citep{qu2025can} maintains prompt-specific estimates of model success, while BOTS \citep{shen2025bots} similarly uses accumulated evidence to guide online task selection. Such history-based methods are most effective when questions are revisited frequently enough for their estimates to become informative. We therefore include MoPPS as a baseline in our smaller-dataset setting, where repeated exposure provides a favorable regime for question-specific histories. Other approaches operate over predefined data categories \citep{chen2025self}, adapt rollout allocation \citep{li2025knapsack}, or identify uninformative samples only after generation \citep{wang2025dump,xiong2025reinforce, foster2026lilo}; these improve training efficiency but do not directly address predicting the utility of previously unseen questions before rollout.

\textbf{Predictive Pre-Rollout Selection.}
The closest prior work to Goldilocks is Prompt Curriculum Learning (PCL) \citep{gao2026prompt}, which, among the methods we consider, is the only approach that learns a prompt-conditioned predictor from model feedback and uses it to score unseen questions \emph{before} generating model rollouts. PCL predicts the model's probability of success and selects questions near a
target difficulty, whereas Goldilocks directly predicts the reward variability
of the rollout group, the quantity motivated by our GRPO analysis. PCL is evaluated on substantially smaller question pools than the large-scale regime considered in our main experiments; we therefore use it as our primary predictive baseline to test whether this form of pre-rollout generalization remains effective as the training pool grows.

\section{Theoretical Motivation: The Goldilocks Zone}\label{sec:grpo_verifiable_reward}

In the context of fine-tuning language models with RL, the objective 
is to maximize the expected reward on a dataset of verifiable answers. 
The standard policy gradient loss is defined as
\(
\mathcal{L}_{PG} = - \mathbb{E}_{s_t, a_t} \left[ \log \pi_\theta(a_t | s_t) 
\hat{A}_t \right],
\)
where $\pi_\theta$ denotes the policy and $\hat{A}_t$ represents the 
estimated advantage. In the context of updating language models, the 
state $s_t$ corresponds to the prompt $q$ combined with the history of 
generated tokens $o_{<t}$, and the action $a_t$ is the next token 
$o_t$. The objective is defined over the generated sequence of 
length $T$:
\[
\mathcal{L}_{\text{PG}}(\theta) = - \mathbb{E}_{q \sim \mathcal{D}, 
o \sim \pi_\theta} \left[ \sum_{t=1}^{T} \log \pi_\theta(o_t \mid q, 
o_{<t}) \hat{A}_t \right],
\]
where $\hat{A}_t$ represents the advantage estimated at token step $t$.

In GRPO, the advantage is computed by 
normalizing rewards across a group of $G$ outputs sampled for a single 
prompt. In the RLVR setting, the reward is assigned at the sequence 
level rather than per token, meaning all tokens within an output $o$ 
share the same advantage estimate. Thus $\hat{A}_t$ is constant across 
token steps $t$ for a given question $q$ and output $o$, and we denote 
it $\hat{A}_{q,o}$. We consider a binary verification reward 
$r_{\text{ver}}=1$ for a correct final answer and $r_{\text{ver}}=0$ 
otherwise.

\paragraph{Advantage computation.} The verification reward follows a 
Bernoulli distribution: $r_{\text{ver}} \sim \text{Bernoulli}(p_q)$. 
The GRPO algorithm computes the advantage $\hat{A}_{q, o}$ by 
standardizing the reward:
\[
    \hat{A}_{q, o} = \frac{r_{\text{ver}} - 
    \text{mean}(r_{\text{ver}})}{\text{std}(r_{\text{ver}})}.
\]
Since $r_{\text{ver}} \sim \text{Bernoulli}(p_q)$, the moments are
\(
    \text{mean}(r_{\text{ver}}) = p_q
\)
and 
\(
    \text{std}(r_{\text{ver}}) = \sqrt{p_q(1 - p_q)}.
\)
Substituting these values, the advantage $\hat{A}_{q, o}$ takes on two 
distinct values depending on the correctness of the output\footnote{This analysis extends to any reward of the form 
$r_{\text{total}} = r_{\text{ver}} + f(t)$, where $f(t)$ converges to 
a constant during training. See Appendix~\ref{app:reward_generalization} 
for a formal proof.}:
\begin{itemize}
    \item \textbf{Correct answer} (probability $p_q$): 
    $\hat{A}_{q, o} = \dfrac{1 - p_q}{\sqrt{p_q (1 - p_q)}}$
    \item \textbf{Incorrect answer} (probability $1 - p_q$): 
    $\hat{A}_{q, o} = \dfrac{-p_q}{\sqrt{p_q (1 - p_q)}}$
\end{itemize}

\paragraph{Zero-variance batches.}
When all $G$ rollouts for a prompt $q$ are simultaneously correct or 
simultaneously incorrect, every rollout receives the same reward, all advantages $\hat{A}_{q,o}$ are zero. The gradient contribution of the entire 
batch therefore vanishes and no parameter update occurs. The 
probability of this zero-variance event is
\begin{equation}
    \mathbb{P}(\text{zero variance}) = p_q^G + (1-p_q)^G,
\label{eq:zero_var_prob}
\end{equation}
which is minimized at $p_q = 1/2$ and 
approaches $1$ as $p_q \to 0$ or $p_q \to 1$. Therefore, prompts near the extremes waste an entire group of rollouts with high 
probability, regardless of model capacity or training budget.

\paragraph{Gradient norm decomposition.} 
We have shown that substituting these advantage values into the 
policy gradient objective and separating the contributions of correct 
and incorrect outcomes gives
\begin{multline}
\left\| \nabla_\theta L_{PG} \right\| = 
\sqrt{p_q (1 - p_q)} \,
\Bigg\|\mathbb{E} \left[ \sum_{t=1}^T \nabla_\theta \log 
\pi_\theta(o_t \mid q, o_{<t}) \,\bigg|\, r_{\text{ver}} = 1\right]
- \mathbb{E} \left[ \sum_{t=1}^T \nabla_\theta \log 
\pi_\theta(o_t \mid q, o_{<t}) \,\bigg|\, r_{\text{ver}} = 0\right] 
\Bigg\|.
\label{eq:grad_derivation_norm}
\end{multline}
The proof is given in Appendix~\ref{app:vanishing_proof}.

\paragraph{Vanishing signal.} 
Equation~\ref{eq:grad_derivation_norm} isolates a
$\sqrt{p_q(1-p_q)}$ prefactor in the gradient norm. Since the remaining
conditional-expectation term is itself question- and model-dependent, this
factor alone does not determine the gradient magnitude. We therefore make its
role precise by deriving a uniform upper bound.
Under the additional assumption that the per-token gradient 
$\|\nabla_\theta \log \pi_\theta(o_t \mid q, o_{<t})\|$ is uniformly 
bounded by some constant $B$ and the context length $T$ is finite, 
this scaling translates into the upper bound
\begin{equation}
\left\|\nabla_\theta L_{PG}\right\| \;\leq\; 
2BT\sqrt{p_q(1-p_q)},
\label{eq:grad_bound}
\end{equation}
where $p_q = p_q(\theta)$ is the verification probability at the 
current parameters. Eq.~\ref{eq:grad_bound} shows that the gradient 
norm is structurally bounded by $\sqrt{p_q(1-p_q)}$, with the constant 
$2BT$ uniform across all prompts at a fixed checkpoint. Therefore, as $p_q \to 0$ or $p_q \to 1$, the gradient norm $\left\|\nabla_\theta \mathcal{L}_{\text{PG}}\right\| \to 0$, leading to vanishing gradients. The formal proof is given in Appendix~\ref{app:vanishing_proof}.

Taken together, Eq.~\ref{eq:zero_var_prob} and 
Eq.~\ref{eq:grad_bound} give two independent reasons why prompts 
near $p_q \in \{0,1\}$ are wasteful: they suppress the gradient 
magnitude and are overwhelmingly likely to produce zero-variance 
batches. Conversely, prompts near $p_q = 1/2$ minimize the probability of zero-variance batches (Eq.~\ref{eq:zero_var_prob}) and maximize the structural upper bound on the gradient norm (Eq.~\ref{eq:grad_bound}). This 
motivates the Goldilocks Selector: select prompts at the edge of the 
model's current solvability, where $\sqrt{p_q(1-p_q)}$ is largest.
We emphasize that the analysis is not intended as a complete
theoretical account of RL dynamics, but rather as a simplified intuition for why intermediate-difficulty questions can provide more informative GRPO updates.

\begin{figure*}[t]
    \centering
    \includegraphics[width=1.0\textwidth]{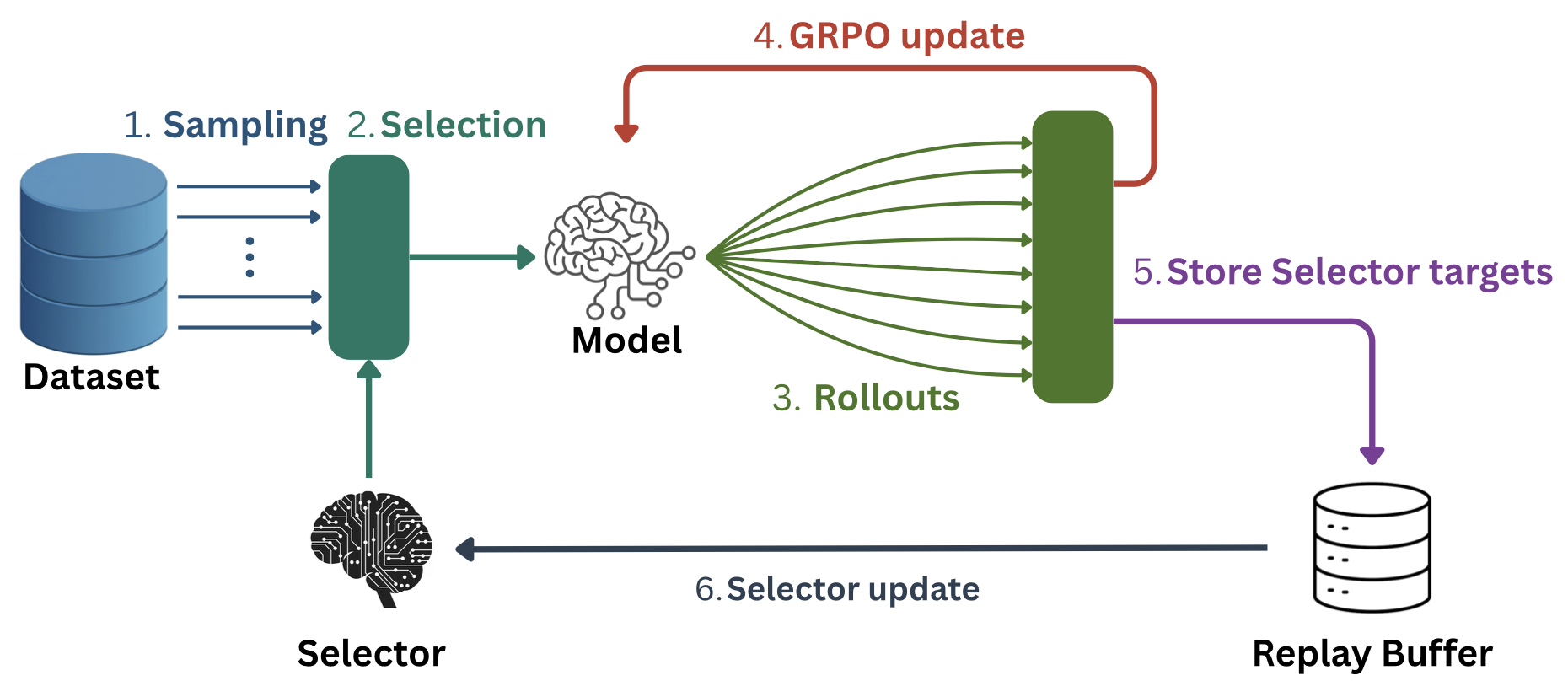}
    \caption{\textbf{Overview of the Goldilocks Framework.} The continuous training loop.}
    \label{fig:framework}
\end{figure*}

\section{Joint Training Procedure} \label{sec:training_proc}
We now detail the overall training loop, illustrated in Figure~\ref{fig:framework}. The process operates in a cycle that alternates between data selection, model optimization via GRPO, and continuous Selector refinement.

\subsection{Data Selection Policy (Steps 1-2)}\label{sec:teacher_data_selection}
During the data selection phase, the Selector samples a candidate pool, denoted as $\mathcal{C}$, consisting of $K_{\text{candidate}}$ examples drawn randomly from the dataset. For each candidate question $q \in \mathcal{C}$, the Selector computes a predicted utility score $\hat{v}_q = f_\phi(q)$. To balance exploration and exploitation, the final training sample $q^*$ is selected using an $\epsilon$-greedy strategy (Algorithm \ref{alg:teacher_selection}):
\begin{itemize}
    \item With probability $\epsilon$, a question is selected uniformly at random from $\mathcal{C}$ to ensure diverse data coverage.
    \item With probability $1 - \epsilon$, the question with the highest predicted utility is selected deterministically, i.e.,
    \[
    q^* = \operatorname*{arg\,max}_{q \in \mathcal{C}} \hat{v}_q.
    \]
\end{itemize}
\subsection{Model Rollouts and Optimization (Steps 3-4)}\label{sec:student_training}
Once the prompt $q^*$ is selected, the model generates a group of $G$ rollouts, $\{o_1, \dots, o_G\}$. The verification reward $r_{\text{ver}}$ is computed for each output, and the advantages are estimated using the group-relative normalization described in Section~\ref{sec:grpo_verifiable_reward}. Based on these advantages, the gradient is computed. To ensure optimization stability, gradients are accumulated over a batch of selected prompts before a parameter update is applied to the model.

\subsection{Selector Refinement (Steps 5-6)}\label{sec:teacher_training}
The Selector is trained online to predict the learning potential of specific prompts based on the model's real-time feedback. Following the model's rollouts for a question $q$, we calculate the empirical success rate $\hat{p}_q$ as the fraction of correct solutions within the generated group. The regression target $y_q$ is then set to the empirical standard deviation of the verification reward \( y_q = \sqrt{\hat{p}_q (1 - \hat{p}_q)}. \)
This tuple $(q, y_q)$ is pushed into a replay buffer $\mathcal{D}_{\text{replay}}$ of capacity $N_{\text{replay}}$, which operates as a sliding window to retain the most recent interaction data.

The Selector updates its parameters $\phi$ periodically as detailed in Algorithm \ref{alg:teacher_update}. After every $M_{\text{update}}$ samples processed by the model, the Selector undergoes a dedicated training phase lasting $E_{\text{Selector}}$ epochs. During this phase, we iterate over the entire replay buffer $\mathcal{D}_{\text{replay}}$, processing the data in mini-batches $\mathcal{B}$ to minimize the Mean Squared Error:
\[
\mathcal{L}_{\mathcal{B}}(\phi) = \frac{1}{|\mathcal{B}|} \sum_{(q, y_q) \in \mathcal{B}} \left( f_\phi(q) - y_q \right)^2.
\]
This objective ensures the Selector continuously aligns its predictions with the model's capabilities.

\section{Experiments and Results} \label{sec:experiments}
To evaluate Goldilocks, we primarily focus on the large-scale
\textbf{OpenMathReasoning} dataset \citep{moshkov2025aimo2}, which comprises over 3 million chain-of-thought problems
that motivates our setting. We compare Goldilocks against
standard GRPO and Prompt Curriculum Learning (PCL) \citep{gao2026prompt},
our closest predictive pre-rollout baseline. Given the substantial computational
cost of long-run GRPO training on OpenMathReasoning, we prioritize extensive
training on this benchmark rather than broad evaluation across many datasets.
As a complementary smaller-scale setting, we additionally train and evaluate on \textbf{Polaris} \citep{Polaris2025}, which contains approximately 50{,}000 questions,
where questions are revisited more frequently. In this regime, we compare
against both PCL and MoPPS \citep{qu2025can}, allowing us to contrast
Goldilocks with a history-based curriculum under conditions where
question-specific histories are more informative.

Across these settings, we evaluate five models spanning different
model families and sizes: OLMo2-1B \citep{olmo20242},
Qwen2.5-1.5B \citep{qwen2.5}, Qwen3-4B-Base \citep{yang2025qwen3}, 
Phi-4-mini-Instruct (4B) \citep{abdin2024phi}, and Qwen2.5-7B \citep{qwen2.5}. Following common practice in
RL-based reasoning research
\citep{shao2024deepseekmath,amani2025rl,chen2025self}, we focus on
mathematical reasoning and train all models with GRPO. All experiments are
implemented using the TRL library \citep{vonwerra2020trl}. Additional dataset,
training, and compute details are provided in
Appendices~\ref{app:datasets} and~\ref{app:implementation}.

\subsection{Experimental Setup}\label{sec:experimental_setup}
For each model--dataset pair, we first tune the GRPO learning rate and
global batch size, and use the best-performing configuration as the reference
for all subsequent comparisons. We then evaluate Goldilocks and the
corresponding baselines using the selected settings. All methods use a rollout group size of $G=16$ in the main experiments, with additional ablations evaluating sensitivity to alternative group sizes.

For all experiments, we use Qwen3-0.6B \citep{yang2025qwen3} as the
Selector. In all experiments, we utilized a fixed computational budget of 8 GPUs.
Vanilla GRPO uses all eight GPUs for model training,
whereas Goldilocks and PCL reserve one GPU for the Selector and use the remaining
seven for the model. Accordingly, if the tuned GRPO global batch size is
$B$, we use a model batch size of $\frac{7}{8}B$ for Goldilocks and PCL,
while keeping the per-GPU model configuration unchanged. This provides a
matched hardware budget and explicitly accounts for the GPU required by the
learned curriculum predictor rather than giving Selector-based methods additional
compute. MoPPS, which does not require a Selector network, follows the corresponding
model's GRPO configuration. We justify this scheme in Appendix~\ref{app:compute_fairness}.

\section{Results}
Our primary experiments on OpenMathReasoning use a long-run training regime to reflect the scale of the dataset and evaluate whether the benefits of adaptive data selection persist over extended training. This extended training regime is particularly important in the large-pool setting, as it allows us to test whether the benefit of Goldilocks persists as the model's capabilities evolve rather than being confined to the early stages of RL training. Our primary comparison focuses on training efficiency relative to tuned GRPO, while we also report validation accuracy at matched training steps. Because the models differ in size and are trained with different batch sizes, they are trained for different numbers of optimization steps under this regime; we therefore report the evaluation step in our tables. Complete hyperparameter searches, selected configurations, and compute accounting are reported in Appendix~\ref{app:implementation}. The total computational cost of these experiments is detailed in Appendix~\ref{app:comp-cost}.

\subsection{OpenMathReasoning: Long-Run Performance and Efficiency}
\label{sec:openmath_results}

Table~\ref{tab:openmath_results} reports both performance and training efficiency on OpenMathReasoning. For each model, all methods are evaluated at the same model-specific training step, enabling a direct comparison under a matched training horizon. We further measure efficiency by identifying the earliest checkpoint whose held-out accuracy matches or exceeds the final GRPO accuracy at that horizon, and report the corresponding percentage reduction in optimization steps. Together, these results capture both the accuracy achieved after extended training and the relative training-step savings required for Goldilocks to reach the corresponding GRPO performance.

\begin{table}[h]
    \centering
    \caption{\textbf{OpenMathReasoning Results.}
    Validation accuracy (\%) is averaged over the final five evaluation
    points at each model's respective training horizon.
    \textbf{Step Reduction} reports the relative reduction in optimization steps
    required for each method to reach GRPO's final accuracy, using five-point
    rolling averages. A $\times$ indicates that the method does not reach
    GRPO's final accuracy within the available training horizon.}
    \label{tab:openmath_results}
    \small
    \setlength{\tabcolsep}{2.5pt}

    \begin{tabular*}{\textwidth}{
        @{\extracolsep{\fill}}lcccccc@{}}
        \toprule
        &
        &
        \multicolumn{3}{c}{\textbf{Accuracy (\%)}} &
        \multicolumn{2}{c}{\textbf{Step Reduction}} \\
        \cmidrule(lr){3-5}
        \cmidrule(lr){6-7}

        \textbf{Model}
        & \textbf{Eval. Step}
        & \textbf{GRPO}
        & \textbf{PCL}
        & \textbf{Goldilocks}
        & \textbf{PCL}
        & \textbf{Goldilocks} \\
        \midrule

        OLMo
        & 40{,}000
        & 16.02
        & 14.47
        & \textbf{20.01}
        & $\times$
        & \textbf{64.0\%} \\

        Qwen2.5-1.5B
        & 13{,}400
        & 40.99
        & 37.46
        & \textbf{44.56}
        & $\times$
        & \textbf{28.4\%} \\

        Qwen3-4B
        & 21{,}000
        & 54.00
        & 53.86
        & \textbf{56.67}
        & $\times$
        & \textbf{41.9\%} \\

        Phi (4B)
        & 16{,}000
        & 40.12
        & 38.39
        & \textbf{44.27}
        & $\times$
        & \textbf{48.8\%} \\

        Qwen2.5-7B
        & 17{,}000
        & 48.27
        & 50.63
        & \textbf{52.61}
        & 32.9\%
        & \textbf{49.4\%} \\

        \bottomrule
    \end{tabular*}
\end{table}

Figure~\ref{fig:main_curves} shows the validation accuracy throughout training for Qwen2.5-7B, comparing standard GRPO, Goldilocks, and PCL. Figures~\ref{fig:wallclock_efficiency} and~\ref{fig:token_efficiency} in Appendix~\ref{app:efficiency_analysis} report this comparison against wall-clock time and cumulative model-generated tokens, respectively. Goldilocks reaches the corresponding GRPO reference accuracy with less elapsed time and fewer generated tokens.

\begin{figure}[h]
\centering
\includegraphics[width=0.6\linewidth]{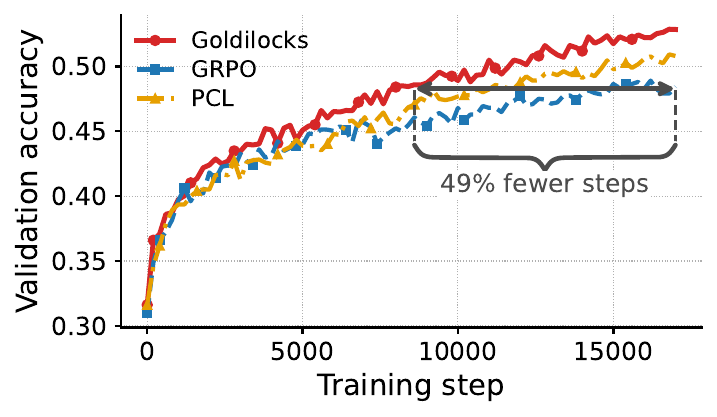}
\caption{\textbf{Long-run reasoning training.} Held-out accuracy on OpenMathReasoning compares curriculum benefits across a large question pool.}
\label{fig:main_curves}
\end{figure}

\subsection{Polaris: Retaining the Benefit When Questions Recur}\label{sec:polaris_results}
The smaller Polaris dataset tests a different explanation for curriculum performance. With more opportunities to revisit each question, a history-based method can build and reuse estimates across a larger fraction of the training pool. This setting is therefore favorable to MoPPS, whose prompt-specific posteriors become more informative as observations accumulate. A benefit confined to sparse-history settings would be less compelling if it disappeared once these histories became readily available.

Goldilocks remains effective on Polaris despite the more frequent question revisitation that favors history-based methods. As shown in Table~\ref{tab:polaris_results}, it achieves higher final accuracy than GRPO and PCL at the common end step for each model. For Qwen2.5-1.5B, Goldilocks reaches 22.26\% final accuracy compared with 21.52\% for GRPO; it also reaches GRPO's final accuracy with 40.0\% fewer optimization steps. Goldilocks surpasses MoPPS in most settings; the exception is Qwen2.5-7B, where MoPPS achieves the highest final accuracy. Goldilocks also reaches the final GRPO performance earlier across all models. These results complement the OpenMathReasoning experiments by showing that Goldilocks remains competitive even when repeated exposure makes historical estimates more informative.

\begin{table}[h]
    \centering
    \caption{\textbf{Training on the smaller Polaris dataset.}
    Validation accuracy (\%) averaged over the final five evaluation points
    at a common end step for each model. All methods are truncated at the
    shortest run's end step. The smaller training pool leads to more frequent
    question revisits, providing a favorable setting for history-based
    selection methods such as MoPPS. \textbf{Step Reduction} reports the relative
    reduction in optimization steps required for each method to reach GRPO's
    final accuracy, using five-point rolling averages. A $\times$ indicates
    that the method does not reach GRPO's final accuracy within the available
    training horizon.}
    \label{tab:polaris_results}
    \small
    \setlength{\tabcolsep}{2.2pt}

    \begin{tabular*}{\textwidth}{
        @{\extracolsep{\fill}}lcccccccc@{}}
        \toprule
        &
        &
        \multicolumn{4}{c}{\textbf{Accuracy (\%)}} &
        \multicolumn{3}{c}{\textbf{Step Reduction}} \\
        \cmidrule(lr){3-6}
        \cmidrule(lr){7-9}

        \textbf{Model}
        & \textbf{End Step}
        & \textbf{GRPO}
        & \textbf{PCL}
        & \textbf{MoPPS}
        & \textbf{Goldilocks}
        & \textbf{PCL}
        & \textbf{MoPPS}
        & \textbf{Goldilocks} \\
        \midrule

        OLMo
        & 29{,}000
        & 21.15
        & 18.65
        & 20.04
        & \textbf{23.12}
        & $\times$
        & $\times$
        & \textbf{51.7\%} \\

        Qwen2.5-1.5B
        & 6{,}000
        & 21.52
        & 19.85
        & 21.56
        & \textbf{22.26}
        & $\times$
        & 23.3\%
        & \textbf{40.0\%} \\

        Qwen3-4B
        & 10{,}000
        & 31.32
        & 30.08
        & 30.95
        & \textbf{31.60}
        & $\times$
        & $\times$
        & \textbf{26.0\%} \\

        Phi (4B)
        & 10{,}000
        & 21.35
        & 23.39
        & 22.65
        & \textbf{25.87}
        & 48.0\%
        & 38.0\%
        & \textbf{78.0\%} \\

        Qwen2.5-7B
        & 17{,}400
        & 33.63
        & 33.94
        & \textbf{38.89}
        & 35.95
        & 16.1\%
        & \textbf{58.6\%}
        & 42.5\% \\

        \bottomrule
    \end{tabular*}
\end{table}

\section{Training Dynamics and Ablations}
\label{sec:ablation}

We first analyze how Goldilocks affects the training dynamics on
OpenMathReasoning using Qwen2.5-7B as a representative model. We then use Qwen2.5-1.5B for controlled ablations.

\subsection{Training Dynamics}
\label{sec:training_dynamics}
We study the training behavior of Goldilocks on Qwen2.5-7B by tracking the
average training reward and the fraction of zero-variance rollout groups
throughout optimization. As shown in Figure~\ref{fig:qwen7b_reward}, PCL keeps
the average reward closer to $0.5$, whereas Goldilocks does not necessarily
match this value as closely. However, Figure~\ref{fig:qwen7b_zero_frac} shows
that Goldilocks produces substantially fewer zero-variance rollout groups than
both GRPO and PCL. This distinction is important: an average reward near $0.5$
does not by itself imply that individual questions have intermediate success
probability, since the same average can arise from a mixture of very easy and
very difficult questions. Goldilocks instead directly targets
$\sqrt{p_q(1-p_q)}$, which favors questions with high within-question reward
variability and suppresses questions near $p_q\in\{0,1\}$. We discuss this
distinction further in Appendix~\ref{app:selection_target_intuition}. We provide analysis of the Selector training dynamics in Appendix~\ref{app:selector_dynamics}.

\begin{figure}[htbp]
    \centering

    \begin{subfigure}[t]{0.48\linewidth}
        \centering
        \includegraphics[width=\linewidth]{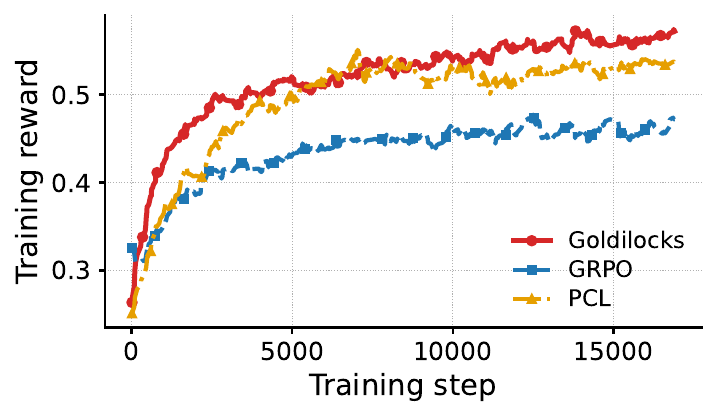}
        \caption{Average training reward.}
        \label{fig:qwen7b_reward}
    \end{subfigure}
    \hfill
    \begin{subfigure}[t]{0.48\linewidth}
        \centering
        \includegraphics[width=\linewidth]{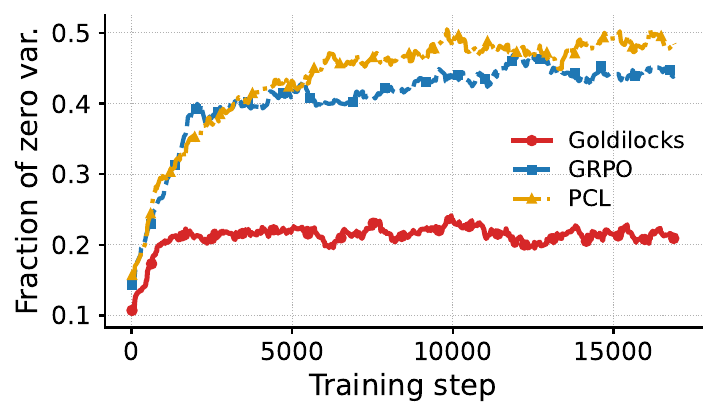}
        \caption{Zero-variance rollout fraction.}
        \label{fig:qwen7b_zero_frac}
    \end{subfigure}

    \caption{\textbf{Training dynamics on Qwen2.5-7B.}}
    \label{fig:qwen7b_training_dynamics}
\end{figure}

\subsection{Ablations}
\label{sec:ablations}

We evaluate the main design choices of Goldilocks using Qwen2.5-1.5B, with
Qwen3-0.6B as the Selector throughout. Table~\ref{tab:ablation_results}
summarizes objective robustness, selection hyperparameters, and rollout-group
size, while additional learning curves are provided in
Appendix~\ref{app:ablation_details}.

\paragraph{Robustness across seeds.}
We repeat the standard GRPO--Goldilocks comparison across four model seeds.
Figure~\ref{fig:seed_robustness} shows the corresponding learning curves,
demonstrating that the improvement is consistent across different
initializations.

\begin{figure}[h]
    \centering
    \includegraphics[width=0.65\linewidth]{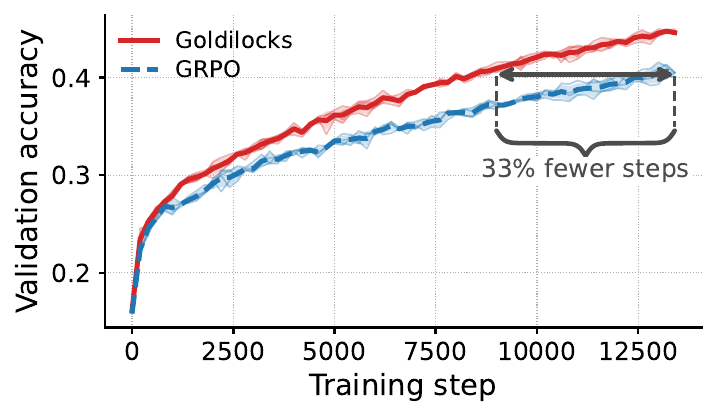}
    \caption{\textbf{Seed robustness on Qwen2.5-1.5B.}
    GRPO and Goldilocks across four model seeds. Solid lines represent the mean performance, while the shaded regions show one standard deviation of the raw validation accuracies across seeds at each evaluation step.}
    \label{fig:seed_robustness}
\end{figure}

\paragraph{Model Objective.}
Goldilocks continues to improve over the corresponding baseline when replacing
the standard GRPO objective with a DAPO-style loss \citep{yu2025dapo} and when
adding entropy regularization. This suggests that the benefit of adaptive
question selection is complementary to changes in the model optimization
objective.

\paragraph{Selection hyperparameters.}
Goldilocks is relatively insensitive to moderate changes in the exploration
probability $\epsilon$ and candidate-pool size $K$ at the final evaluation
horizon. Without exploration, the Selector only
receives feedback on questions favored by its current predictions, which can
reinforce selection errors and reduce coverage of underestimated questions.
A small amount of random exploration mitigates this feedback loop. The full
learning curves for $\epsilon$ and $K$ are in Appendix~\ref{app:ablation_details}.

\paragraph{Generations per question.}
We additionally vary the rollout-group size $G$ for both GRPO and Goldilocks.
Across the tested settings, Goldilocks retains its advantage, indicating that
its benefit is not specific to the default number of generations per question.

\begin{table}[t]
\centering
\small
\setlength{\tabcolsep}{3pt}
\caption{\textbf{Ablations on OpenMathReasoning.}
Validation accuracy (\%) with Qwen2.5-1.5B, averaged over the final five
evaluations at the listed training horizon, shared by all runs in each comparison.
The reference configuration uses $G=16$; Goldilocks additionally uses
$K=8$ and $\epsilon=0.2$.
Seed results are mean $\pm$ sample standard deviation across four seeds;
their step reduction is computed from the mean curves.
\textbf{Step Red.} reports the reduction in optimization steps required for
Goldilocks to reach GRPO's final rolling-average accuracy at the same horizon.}
\label{tab:ablation_results}

\begin{tabular*}{\textwidth}{@{\extracolsep{\fill}}lllccc@{}}
\toprule
& & &
\multicolumn{2}{c}{\textbf{Accuracy (\%)}} &
\multicolumn{1}{c}{\textbf{Efficiency}} \\
\cmidrule(lr){4-5}
\cmidrule(lr){6-6}
\textbf{Group}
& \textbf{Setting}
& \textbf{Steps}
& \textbf{GRPO}
& \textbf{Goldilocks}
& \textbf{Step Red.} \\
\midrule

\textbf{Main Results}
& Standard
& 13{,}400
& 40.99
& \textbf{44.56}
& 28.4\% \\

\midrule

\multirow{2}{*}{\shortstack{\textbf{Objective}\\\textbf{Robustness}}}

& DAPO Loss
& 13{,}400
& 39.50
& \textbf{43.66}
& 38.8\% \\

& \shortstack[l]{Entropy Reg.} $\beta=3\times10^{-4}$
& 13{,}400
& 40.28
& \textbf{44.24}
& 29.9\% \\

\midrule

\textbf{Seeds}
& Four seeds
& 13{,}400
& $40.33 \pm 0.58$
& $\mathbf{44.46 \pm 0.30}$
& 32.8\% \\

\midrule

\multirow{6}{*}{\shortstack{\textbf{Selection}\\\textbf{Ablations}}}
& $\epsilon=0$
& 13{,}400
& 40.99
& \textbf{44.26}
& 31.3\% \\

& $\epsilon=0.2$
& 13{,}400
& 40.99
& \textbf{44.56}
& 28.4\% \\

& $\epsilon=0.5$
& 13{,}400
& 40.99
& \textbf{42.43}
& 11.9\% \\

\cmidrule(lr){2-6}

& $K=8$
& 13{,}400
& 40.99
& \textbf{44.56}
& 28.4\% \\

& $K=16$
& 13{,}400
& 40.99
& \textbf{43.76}
& 29.9\% \\

& $K=32$
& 13{,}400
& 40.99
& \textbf{44.02}
& 29.9\% \\

\midrule

\multirow{3}{*}{\shortstack{\textbf{Rollout}\\\textbf{Ablation}}}
& $G=8$
& 13{,}400
& 40.87
& \textbf{44.38}
& 32.8\% \\

& $G=16$
& 13{,}400
& 40.99
& \textbf{44.56}
& 28.4\% \\

& $G=32$
& 13{,}400
& 38.36
& \textbf{42.75}
& 41.8\% \\

\bottomrule
\end{tabular*}
\end{table}

\section{Limitations and Future Work}\label{sec:conclusion}
\paragraph{Larger models.}
Our experiments study models up to 7B parameters and use a 0.6B Selector throughout. They do not establish how the approach scales to substantially larger models, whose evolving capabilities may be harder for a small predictor to track. Future work should evaluate larger Selector and model combinations, measure how prediction quality changes with the capacity gap, and determine whether curriculum gains persist after all Selector costs are included.

\paragraph{Datasets beyond mathematics.}
OpenMathReasoning and Polaris differ in scale and question reuse, but both concern mathematical reasoning with verifiable answers. Generalization to other domains remains untested. Code generation, theorem proving, and multi-step tasks offer natural next settings, with different verifiers, reward distributions, and computational costs. Extending the selection target to these datasets, particularly when feedback is graded or noisy rather than binary, is an important direction for future work.

\clearpage

\bibliographystyle{plainnat}
\bibliography{bibliography}

\clearpage
\appendix
\section{Algorithms}\label{app:algorithms}
\begin{algorithm}[ht]
   \caption{Goldilocks Selector: Data Selection Policy}
   \label{alg:teacher_selection}
\begin{algorithmic}[1]
   \STATE {\bfseries Input:} Dataset $\mathcal{D}$, Selector network $f_\phi$, Candidate Size $K$, Exploration $\epsilon$.
   \STATE {\bfseries Output:} Selected question $q^*$.
   
   \vspace{0.2cm}
   \FUNCTION{SelectQuery($\phi$, $\mathcal{D}$)}
       \STATE $\mathcal{C} \leftarrow$ Sample $K$ random indices from $\mathcal{D}$
       \STATE $\mathcal{V} \leftarrow \emptyset$
       \FOR{each candidate $q \in \mathcal{C}$}
           \STATE $\hat{v}_q \leftarrow f_\phi(q)$ \COMMENT{Predict learning utility}
           \STATE $\mathcal{V} \leftarrow \mathcal{V} \cup \{\hat{v}_q\}$
       \ENDFOR
       
       \STATE $r \sim \text{Uniform}(0, 1)$
       \IF{$r < \epsilon$}
           \STATE $q^* \leftarrow \text{Uniform}(\mathcal{C})$ \COMMENT{Ensure coverage ($\epsilon$-greedy)}
        \ELSE
           \STATE $q^* \leftarrow \operatorname*{arg\,max}_{q \in \mathcal{C}} \hat{v}_q$ \COMMENT{Greedy maximization}
       \ENDIF
       \STATE \textbf{return} $q^*$
   \ENDFUNCTION
\end{algorithmic}
\end{algorithm}
\begin{algorithm}[h!]
   \caption{Goldilocks Selector: Online Refinement}
   \label{alg:teacher_update}
\begin{algorithmic}[1]
   \STATE {\bfseries Input:} Question $q$, Rollouts $\{o_1, ..., o_G\}$, Selector network $f_\phi$, Replay Buffer $\mathcal{D}_{\text{replay}}$, Capacity $N_{\text{replay}}$.
   
   \vspace{0.2cm}
   \FUNCTION{Update Selector($q$, $\{o_i\}_{i=1}^G$)}
       \STATE $\hat{p}_q \leftarrow \frac{1}{G} \sum_{i=1}^G \mathbb{I}(o_i \text{ is correct})$ \COMMENT{Empirical success rate}
       \STATE $y_q \leftarrow \sqrt{\hat{p}_q (1 - \hat{p}_q)}$ \COMMENT{Compute target utility}
       
       \STATE Push $(q, y_q)$ to $\mathcal{D}_{\text{replay}}$
       \IF{$|\mathcal{D}_{\text{replay}}| > N_{\text{replay}}$}
           \STATE Remove oldest sample from $\mathcal{D}_{\text{replay}}$ \COMMENT{Sliding window}
       \ENDIF
       
       \STATE \COMMENT{Periodic Update Trigger (Step 6)}
       \IF{Update Condition Met}
          \STATE Shuffle $\mathcal{D}_{\text{replay}}$ and create batches $\mathcal{B}$
           \FOR{epoch $1$ to $E_{\text{Selector}}$}
                \FOR{batch $B \in \mathcal{B}$}
                    \STATE $\mathcal{L} \leftarrow \frac{1}{|B|} \sum_{(q, y) \in B} (f_\phi(q) - y)^2$
                    \STATE $\phi \leftarrow \phi - \eta \nabla_\phi \mathcal{L}$ \COMMENT{Gradient descent step}
                \ENDFOR
           \ENDFOR
       \ENDIF
   \ENDFUNCTION
\end{algorithmic}
\end{algorithm}

\clearpage

\section{Generalization to Composite Rewards}
\label{app:reward_generalization}

\begin{assumption}[Convergence of auxiliary reward]\label{as:aux_converge}
Let $r_{\text{total}} = r_{\text{ver}} + f(t)$, where $f(t)$ is any reward component depending on training step $t$. There exists an iteration threshold $N \in \mathbb{N}$ such that for all $t \ge N$, $f(t) = C$ for some constant $C \in \mathbb{R}$.
\end{assumption}

\begin{proposition}
Under Assumption~\ref{as:aux_converge}, for sufficiently large $t$, the GRPO advantage computed from $r_{\text{total}}$ is identical to that computed from $r_{\text{ver}}$ alone.
\end{proposition}

\begin{proof}
For $t \ge N$, we have $r_{\text{total}} = r_{\text{ver}} + C$. Since $r_{\text{ver}} \sim \text{Bernoulli}(p_q)$:
\[
    \text{mean}(r_{\text{total}}) = p_q + C, \quad \text{std}(r_{\text{total}}) = \sqrt{p_q(1-p_q)}.
\]
The advantage is:
\begin{align*}
    \hat{A}_q 
    &= \frac{r_{\text{total}} - \text{mean}(r_{\text{total}})}{\text{std}(r_{\text{total}})} \\
    &= \frac{(r_{\text{ver}} + C) - (p_q + C)}{\sqrt{p_q(1-p_q)}} \\
    &= \frac{r_{\text{ver}} - p_q}{\sqrt{p_q(1-p_q)}},
\end{align*}
which is identical to the advantage computed from $r_{\text{ver}}$ alone.
\end{proof}

\clearpage

\section{Proofs for GRPO with Verifiable Reward} \label{app:vanishing_proof}

In this appendix, we provide the formal derivations and proofs for the policy gradient properties under the GRPO framework with binary verifiable rewards. First, we derive the decomposition of the policy gradient norm. Then, we use this decomposition to prove the exact scaling identity and the structural upper bound of the gradient signal.

\subsection{Derivation of the Gradient Norm Decomposition}

Recall that the gradient of the standard policy objective for a fixed prompt $q$ is the expected product of the trajectory score $S_\theta(o)$ and the sequence-level advantage $\hat{A}_{q, o}$:
\begin{equation}
\nabla_\theta \mathcal{L}_{\text{PG}} = - \mathbb{E}_{o \sim \pi_\theta} \left[ S_\theta(o) \hat{A}_{q, o} \right],
\end{equation}
where the trajectory score is $S_\theta(o) = \sum_{t=1}^T \nabla_\theta \log \pi_\theta(o_t \mid q, o_{<t})$. 

As established in Section~\ref{sec:grpo_verifiable_reward}, the advantage is the standardized binary reward. Since $r_{\text{ver}} \sim \text{Bernoulli}(p_q)$, we substitute the advantage:
\[
\hat{A}_{q, o} = \frac{r_{\text{ver}} - p_q}{\sqrt{p_q(1-p_q)}}
\]
Plugging this into the gradient expectation yields:
\[
\nabla_\theta \mathcal{L}_{\text{PG}} = - \mathbb{E}_{o \sim \pi_\theta} \left[ S_\theta(o) \frac{r_{\text{ver}} - p_q}{\sqrt{p_q(1-p_q)}} \right]
\]
Because $r_{\text{ver}} \in \{0, 1\}$, we apply the law of total expectation to split this into two conditional branches weighted by their respective probabilities, $p_q$ and $(1-p_q)$:
\[
\nabla_\theta \mathcal{L}_{\text{PG}} = - \left( p_q \mathbb{E}\left[ S_\theta(o) \frac{1 - p_q}{\sqrt{p_q(1-p_q)}} \,\bigg|\, r_{\text{ver}}=1 \right] + (1 - p_q) \mathbb{E}\left[ S_\theta(o) \frac{0 - p_q}{\sqrt{p_q(1-p_q)}} \,\bigg|\, r_{\text{ver}}=0 \right] \right)
\]
Since $p_q$ is constant with respect to the inner expectation, we can factor the scalars out. Notice that the coefficients for both terms evaluate to the exact same magnitude:
\[
\frac{p_q(1 - p_q)}{\sqrt{p_q(1-p_q)}} = \sqrt{p_q(1-p_q)}
\]
Substituting this simplified coefficient back into the expression gives:
\[
\nabla_\theta \mathcal{L}_{\text{PG}} = - \left( \sqrt{p_q(1-p_q)} \mathbb{E}\left[ S_\theta(o) \,\big|\, r_{\text{ver}}=1 \right] - \sqrt{p_q(1-p_q)} \mathbb{E}\left[ S_\theta(o) \,\big|\, r_{\text{ver}}=0 \right] \right)
\]
Factoring out the $\sqrt{p_q(1-p_q)}$ scalar and taking the norm of both sides yields the gradient norm decomposition:
\begin{equation}
\left\| \nabla_\theta \mathcal{L}_{\text{PG}} \right\| = \sqrt{p_q(1-p_q)} \Bigg\| \mathbb{E}\left[ S_\theta(o) \,\bigg|\, r_{\text{ver}}=1 \right] - \mathbb{E}\left[ S_\theta(o) \,\bigg|\, r_{\text{ver}}=0 \right] \Bigg\|
\label{eq:grad_derivation_norm_app}
\end{equation}
\subsection{Derivation of the Gradient Upper Bound}

Starting from Eq.~\ref{eq:grad_derivation_norm_app}, the policy-gradient norm
for a fixed question $q$ can be written as
\[
\left\|\nabla_\theta \mathcal{L}_{\mathrm{PG}}\right\|
=
\sqrt{p_q(1-p_q)}
\left\|
\mathbb{E}[S_\theta(o)\mid r_{\mathrm{ver}}=1]
-
\mathbb{E}[S_\theta(o)\mid r_{\mathrm{ver}}=0]
\right\|,
\]
where
\[
S_\theta(o)
=
\sum_{t=1}^{T}
\nabla_\theta \log \pi_\theta(o_t\mid q,o_{<t})
\]
is the trajectory score.

Assume that the per-token score is uniformly bounded by a constant $B$,
\[
\left\|
\nabla_\theta \log \pi_\theta(o_t\mid q,o_{<t})
\right\|
\leq B.
\]
By the triangle inequality, the norm of the trajectory score is then bounded by
\[
\|S_\theta(o)\|
\leq
\sum_{t=1}^{T}
\left\|
\nabla_\theta \log \pi_\theta(o_t\mid q,o_{<t})
\right\|
\leq BT.
\]
The same bound therefore holds for the conditional expectations,
\[
\left\|
\mathbb{E}[S_\theta(o)\mid r_{\mathrm{ver}}=1]
\right\|
\leq BT,
\qquad
\left\|
\mathbb{E}[S_\theta(o)\mid r_{\mathrm{ver}}=0]
\right\|
\leq BT.
\]

Applying the triangle inequality once more gives
\begin{align}
\left\|\nabla_\theta \mathcal{L}_{\mathrm{PG}}\right\|
&\leq
\sqrt{p_q(1-p_q)}
\left(
\left\|\mathbb{E}[S_\theta(o)\mid r_{\mathrm{ver}}=1]\right\|
+
\left\|\mathbb{E}[S_\theta(o)\mid r_{\mathrm{ver}}=0]\right\|
\right)
\\
&\leq
2BT\sqrt{p_q(1-p_q)}.
\end{align}

Thus, under a uniform bound on the per-token score, the policy-gradient norm is
upper bounded by a term proportional to $\sqrt{p_q(1-p_q)}$. This bound
vanishes as $p_q\rightarrow 0$ or $p_q\rightarrow 1$ and is largest at
$p_q=1/2$, providing the intuition behind prioritizing questions of
intermediate difficulty.
\clearpage

\section{Dataset Details}\label{app:datasets}
For evaluation, we construct fixed validation sets of 5,000 problems for both OpenMathReasoning and Polaris, sampled randomly from their respective datasets and kept strictly disjoint from the corresponding training sets. The same validation subset is used across all methods and runs for each dataset to ensure consistent and comparable evaluation.
\clearpage

\section{Hyperparameter Selection and Implementation}
\label{app:implementation}

\subsection{Model Configuration}
\label{app:training_config}
\label{app:compute_fairness}

For each Model--dataset pair, we tune GRPO over batch sizes
$\{256,512,768,1024,2048\}$ and learning rates
$\{10^{-5},10^{-6},3.33\times10^{-6},6.66\times10^{-6}\}$,
and select the best-performing configuration using validation accuracy on same wall time.
The selected settings depend on both the Model and the training
dataset. We then use the selected Model learning rate, per-device batch
size, and gradient-accumulation factor for Goldilocks and the corresponding
curriculum baselines, with the effective batch adjusted for the allocation
of GPUs between Selector network and model described below.
Table~\ref{tab:tuning_protocol} summarizes the search space, and
Table~\ref{tab:selected_student_configs} reports the selected GRPO settings
for every Model--dataset pair.

\begin{table}[htbp]
\centering
\small
\caption{\textbf{GRPO hyperparameter search and shared Model configuration.}
Batch size and learning rate are tuned separately for each Model--dataset
pair. Batch sizes count generated completions per optimization step.}
\label{tab:tuning_protocol}
\begin{tabular}{p{0.43\linewidth}p{0.46\linewidth}}
\toprule
\textbf{Setting} & \textbf{Value} \\
\midrule
Batch-size search space & 256, 512, 768, 1{,}024, 2{,}048 \\
Learning-rate search space & $10^{-5}$, $10^{-6}$, $3.33\times10^{-6}$, $6.66\times10^{-6}$ \\
Selection criterion & Validation accuracy \\
Generations per question $G$ & 16 \\
\bottomrule
\end{tabular}
\end{table}

\begin{table}[htbp]
\centering
\small
\setlength{\tabcolsep}{3pt}
\caption{\textbf{Selected GRPO Model hyperparameters.}
Batch size is the total number of completions per optimization step across
all eight Model GPUs. All compared methods use the listed Model learning
rate; Goldilocks and PCL use $7/8$ of the listed batch size when one GPU is
reserved for the Selector.}
\label{tab:selected_student_configs}
\begin{tabular*}{\textwidth}{@{\extracolsep{\fill}}lcccc@{}}
\toprule
& \multicolumn{2}{c}{\textbf{OpenMathReasoning}}
& \multicolumn{2}{c}{\textbf{Polaris}} \\
\cmidrule(lr){2-3}
\cmidrule(lr){4-5}
\textbf{Model}
& \textbf{Batch size} & \textbf{Learning rate}
& \textbf{Batch size} & \textbf{Learning rate} \\
\midrule
OLMo2-1B & 768 & $10^{-5}$ & 1{,}024 & $10^{-5}$ \\
Qwen2.5-1.5B & 1{,}024 & $10^{-5}$ & 1{,}024 & $10^{-5}$ \\
Qwen3-4B-Base & 256 & $3.33\times10^{-6}$ & 256 & $3.33\times10^{-6}$ \\
Phi-4-mini-Instruct & 256 & $10^{-5}$ & 256 & $10^{-5}$ \\
Qwen2.5-7B & 256 & $3.33\times10^{-6}$ & 256 & $3.33\times10^{-6}$ \\
\bottomrule
\end{tabular*}
\end{table}

The effective completion batch is
$B_{\mathrm{comp}}=n_{\mathrm{GPU}}b_{\mathrm{device}}a_{\mathrm{accum}}$,
where $n_{\mathrm{GPU}}$ is the number of Model GPUs,
$b_{\mathrm{device}}$ is the per-device completion batch, and
$a_{\mathrm{accum}}$ is the gradient-accumulation factor.
Each update therefore contains $B_{\mathrm{comp}}/G$ prompt groups.
The selected GRPO configurations use eight Model GPUs and a per-device
batch size of four. Under the fixed eight-GPU budget in
Section~\ref{sec:experimental_setup}, Goldilocks and PCL allocate one GPU to
the Selector and seven to the Model. Keeping the per-device batch and
accumulation factor fixed therefore gives these methods $7/8$ of the GRPO
completion batch. This accounts for the Selector's hardware cost while
preserving the selected Model learning rate and per-device configuration.

\subsection{Selector Architecture}\label{app:teacher_arch}
The Selector functions as a value predictor built upon a trainable language model backbone. For a given input question $q$, we extract the final layer embeddings and apply mean pooling to obtain a fixed-size vector representation. This vector is then passed through a linear projection layer. To strictly enforce the valid range of the standard deviation term (which lies between $0$ and $0.5$ for a Bernoulli variable), we apply a scaled sigmoid activation, i.e., 
\[
f_\phi(q) = 0.5 \cdot \sigma(\mathbf{w}^T \cdot \text{MeanPool}(\text{Embed}(q)) + b),
\]
where $f_\phi(q)$ represents the predicted utility of question $q$. The parameters $\phi$ encompass both the projection head and the pretrained language model, allowing the Selector to learn deep semantic features relevant to difficulty estimation. For a detailed exploration of alternative Selector backbones and projection strategies, please refer to Appendix~\ref{app:other_teacher_arc}.

\subsection{Selector Configuration}\label{app:other_teacher_arc}
Every learned curriculum predictor uses Qwen3-0.6B, independent of the Model architecture or training dataset. Goldilocks fine-tunes this backbone to predict the standard deviation of binary correctness rewards from question text (Section~\ref{app:teacher_arch}). Its Selector configuration is shared across the main experiments (Table~\ref{tab:teacher_settings}); the ablations vary only the specified component (Appendix~\ref{app:ablation_details}).

\begin{table}[htbp]
\centering\small
\caption{Goldilocks Selector configuration.}
\label{tab:teacher_settings}
\begin{tabular}{lc}
\toprule
Setting & Value \\
\midrule
Backbone & Qwen3-0.6B \\
Representation & Masked mean of final hidden states \\
Output & Scaled sigmoid in $[0,1/2]$ \\
Adaptation & Full-parameter fine-tuning \\
Optimizer / learning rate & AdamW / $10^{-5}$ \\
Selector minibatch size & 32 \\
Passes over replay buffer per update & 4 \\
Replay capacity & 128 questions \\
Random warm-up & 64 question selections \\
Candidate count $K$ & 8 \\
Exploration probability $\epsilon$ & 0.2 \\
Prediction target / loss & Reward standard deviation / squared error \\
\bottomrule
\end{tabular}
\end{table}

Selector updates require at least one complete replay minibatch. The bounded buffer emphasizes recent Model behavior as the policy improves; all Selector labels come from Model rollouts already used for training. The number of completions used to form each Selector target matches the Model's rollout group size, including in the generation-count ablations (Appendix~\ref{app:ablation_details}).

\subsection{Baseline Details}\label{app:baseline_details}
GRPO uses the standard sampler over the shuffled training dataset.
PCL~\citep{gao2026prompt} learns a success-probability predictor with the same
Qwen3-0.6B backbone as Goldilocks and prioritizes questions near a target
success rate of $0.5$. MoPPS~\citep{qu2025can} maintains a Beta posterior for
each question and selects prompts by posterior sampling.

For PCL and MoPPS, we use the exact method-specific hyperparameters reported
in their original papers, except that we additionally tune PCL's Selector
learning rate to ensure a well-tuned baseline implementation. The selected
Selector learning rate is $10^{-6}$ across all Model--dataset pairs; the
remaining baseline-specific hyperparameters retain their published values.
Model optimization follows the GRPO-selected learning rate and batch
configuration in Table~\ref{tab:selected_student_configs}, including the
GPU-allocation adjustment for learned Selectors.

\clearpage

\section{Compute and Training Efficiency}
\label{app:efficiency_analysis}

To complement the optimization-step comparisons in the main text, we evaluate
training efficiency using two additional measures: wall-clock time and the
number of Model-generated tokens. Wall-clock time captures the end-to-end
training cost, including Selector scoring and optimization for learned curriculum
methods, while generated tokens isolate the dominant cost of Model rollout
generation. These views allow us to distinguish improvements in sampling
efficiency from differences in system-level overhead. The following comparisons use Qwen2.5-7B on OpenMathReasoning. GRPO's last available point within that budget
defines the reference accuracy.

\paragraph{Wall-clock efficiency.}
Figure~\ref{fig:wallclock_efficiency} reports accuracy as a function
of elapsed training time. All methods are evaluated under the same hardware
allocation described in Section~\ref{sec:experimental_setup}. For Goldilocks
and PCL, the reported time includes Selector inference and Selector updates.

\begin{figure}[htbp]
\centering
\includegraphics[width=0.52\linewidth]{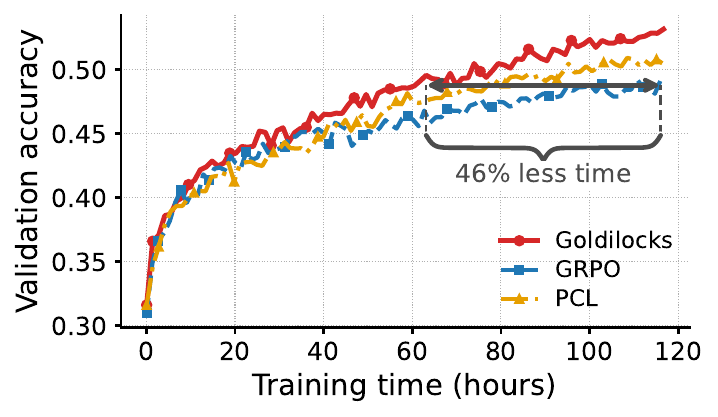}
\caption{\textbf{Accuracy versus wall-clock time.}
Held-out accuracy of Qwen2.5-7B on OpenMathReasoning throughout training for
GRPO, PCL, and Goldilocks.}
\label{fig:wallclock_efficiency}
\end{figure}

\paragraph{Generated-token efficiency.}
Figure~\ref{fig:token_efficiency} reports accuracy against the
cumulative number of tokens generated by the Model during training. This
comparison measures how effectively each method converts rollout generation
into improvements in held-out performance, independent of Selector-side
computational overhead.

\begin{figure}[h]
\centering
\includegraphics[width=0.52\linewidth]{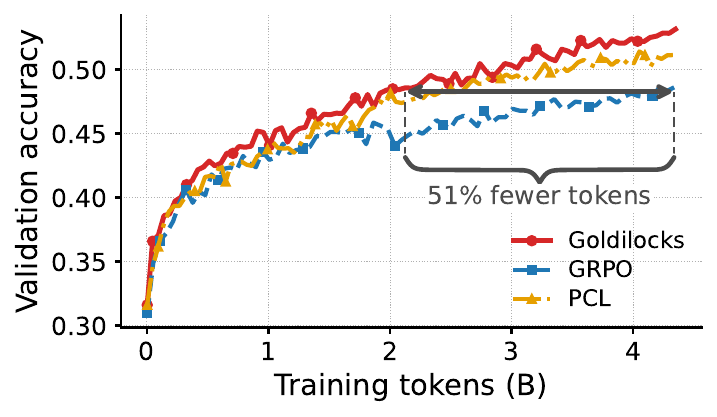}
\caption{\textbf{Accuracy versus generated-token budget.}
Held-out accuracy of Qwen2.5-7B on OpenMathReasoning as a function of cumulative
Model-generated tokens for GRPO, PCL, and Goldilocks.}
\label{fig:token_efficiency}
\end{figure}

Together, the step-, wall-clock-, and token-based evaluations provide
complementary measures of efficiency. Optimization steps measure learning
progress at the algorithmic level, generated tokens capture rollout efficiency,
and wall-clock time reflects the end-to-end cost of running each method.
\clearpage

\section{Selector Training Dynamics}
\label{app:selector_dynamics}

We further analyze the behavior of the Goldilocks Selector during training.
In these experiments, each Model optimization step corresponds to four Selector
optimization steps. Consequently, the Selector-step axis in the following
figures spans approximately four times as many steps as the Model-training axes
reported elsewhere in the paper.

First, we measure how accurately the Selector tracks the reward variability
observed from the Model's rollouts. For each question, the Selector predicts a
utility score $f_\phi(q)$ before generation. After obtaining the rollout group,
we compute the realized target
\[
y_q = \sqrt{\hat p_q(1-\hat p_q)},
\]
and evaluate the absolute prediction error $|f_\phi(q)-y_q|$.
Figure~\ref{fig:selector_prediction_error} shows that this error decreases over
training, indicating that the Selector progressively adapts to the Model's
evolving capabilities. We additionally report the error on zero-variance
groups, for which $y_q=0$, to measure how well the Selector identifies
questions expected to provide little relative learning signal.

\begin{figure}[h]
    \centering
    \includegraphics[width=0.45\linewidth]
    {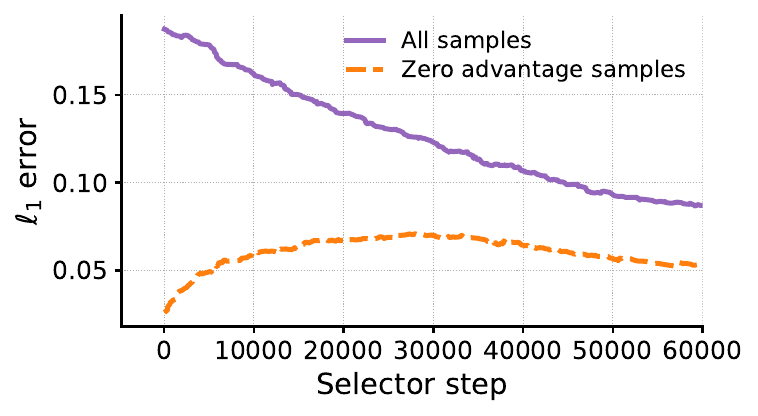}
    \caption{\textbf{Selector prediction error on OpenMathReasoning.}
    Absolute error between predicted and realized reward-variability targets
    throughout training. The x-axis reports Selector optimization steps; each
    Model optimization step corresponds to four Selector steps.}
    \label{fig:selector_prediction_error}
\end{figure}

\paragraph{Diversity of Selector predictions.}
We also track the distribution of utility scores predicted by the Selector
throughout training. Figure~\ref{fig:selector_prediction_distribution} reports
the mean predicted score together with one standard deviation. The persistent
spread in predicted utilities shows that the Selector does not collapse to a
nearly constant scoring function as training progresses; instead, it continues
to distinguish questions according to their predicted reward variability as
the Model evolves.

\begin{figure}[h]
    \centering
    \includegraphics[width=0.45\linewidth]
    {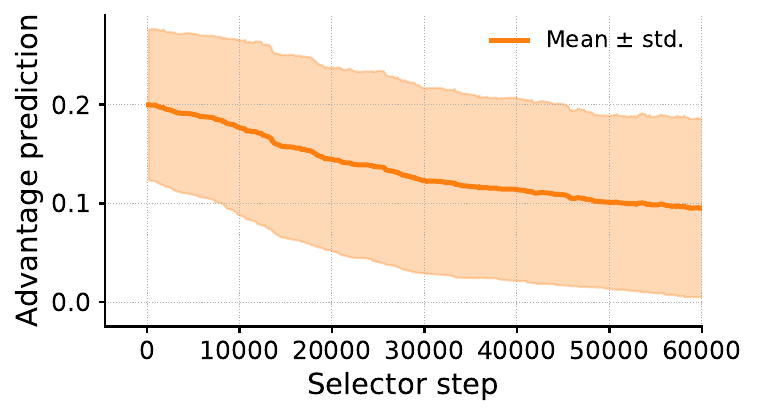}
    \caption{\textbf{Distribution of Goldilocks Selector scores.}
    Mean predicted utility with $\pm$ one standard deviation over Selector
    optimization steps.}
    \label{fig:selector_prediction_distribution}
\end{figure}
\clearpage

\section{Detailed Ablations}
\label{app:ablation_details}
\label{app:ablation_plots}

All ablations use Qwen2.5-1.5B on OpenMathReasoning with Qwen3-0.6B as the
Selector. Unless varied explicitly, we use the reference configuration
$\epsilon=0.2$, $K=8$, and $G=16$ (Table~\ref{tab:teacher_settings}), together
with the Model batch size and learning rate selected for GRPO
(Table~\ref{tab:tuning_protocol}).

\paragraph{Optimization objective.}
We first test whether the benefit of Goldilocks depends on the particular
model optimization objective. In addition to standard GRPO, we consider a
DAPO-style loss~\citep{yu2025dapo} and GRPO with entropy regularization
$\beta=3\times10^{-4}$. In each case, GRPO and Goldilocks use the same
objective, isolating the effect of adaptive question selection.
Figure~\ref{fig:ablation_objective} shows the corresponding validation accuracy
and zero-variance fraction.

\begin{figure}[h]
    \centering
    \begin{subfigure}[t]{0.48\linewidth}
        \centering
        \includegraphics[width=\linewidth]
        {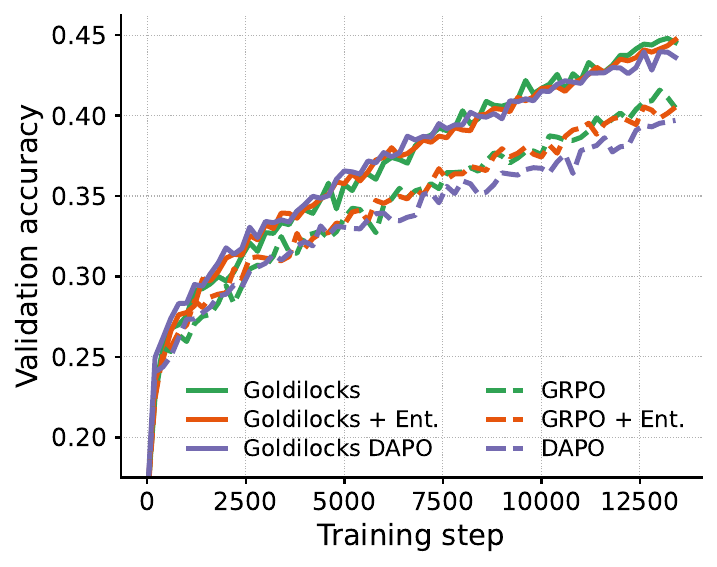}
        \caption{Validation accuracy.}
        \label{fig:ablation_objective_eval}
    \end{subfigure}
    \hfill
    \begin{subfigure}[t]{0.48\linewidth}
        \centering
        \includegraphics[width=\linewidth]
        {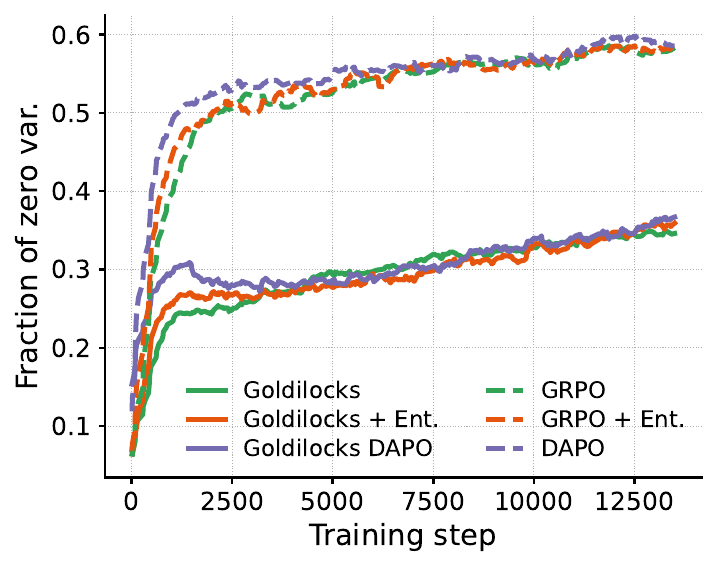}
        \caption{Zero-variance fraction.}
        \label{fig:ablation_objective_zero_frac}
    \end{subfigure}
    \caption{\textbf{Optimization-objective ablation.}
    GRPO and Goldilocks under standard GRPO, DAPO, and entropy-regularized
    objectives.}
    \label{fig:ablation_objective}
\end{figure}

\clearpage

\paragraph{Exploration probability.}
We vary $\epsilon\in\{0,0.2,0.5\}$ to study the role of random exploration in
the Selector. Final performance is relatively insensitive to moderate changes
in $\epsilon$: larger values increase randomness in the selected questions,
whereas smaller values make the Selector rely more heavily on its own
predictions and therefore on the subset of questions it already favors.
Figure~\ref{fig:ablation_epsilon} reports validation accuracy together with
the corresponding training-reward and zero-variance dynamics.

\begin{figure}[h]
    \centering

    \begin{subfigure}[t]{0.48\linewidth}
        \centering
        \includegraphics[width=\linewidth]
        {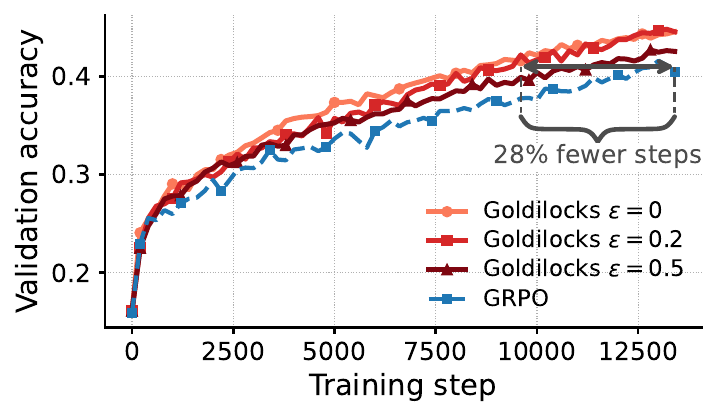}
        \caption{Validation accuracy.}
        \label{fig:ablation_epsilon_eval}
    \end{subfigure}

    \par\medskip

    \begin{subfigure}[t]{0.48\linewidth}
        \centering
        \includegraphics[width=\linewidth]
        {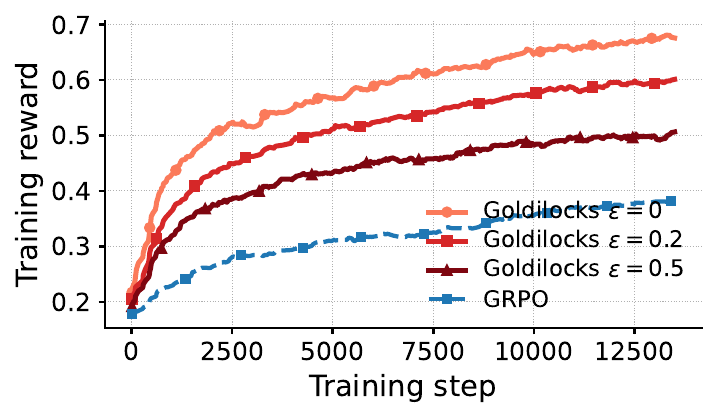}
        \caption{Average training reward.}
        \label{fig:ablation_epsilon_reward}
    \end{subfigure}
    \hfill
    \begin{subfigure}[t]{0.48\linewidth}
        \centering
        \includegraphics[width=\linewidth]
        {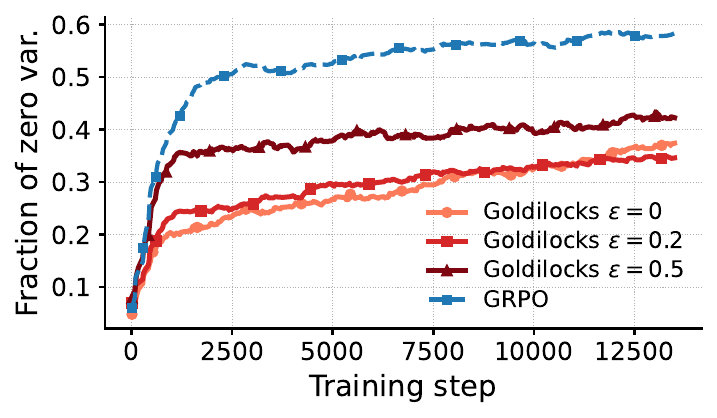}
        \caption{Zero-variance fraction.}
        \label{fig:ablation_epsilon_zero_frac}
    \end{subfigure}

    \caption{\textbf{Exploration-probability ablation.}
    Training dynamics for $\epsilon\in\{0,0.2,0.5\}$.}
    \label{fig:ablation_epsilon}
\end{figure}

\clearpage

\paragraph{Candidate-pool size.}
We vary $K\in\{8,16,32\}$ to study sensitivity to the number of candidate
questions scored before each selection. Increasing $K$ gives the Selector more
opportunities to identify informative questions, while also increasing
Selector inference cost. Across the tested range, performance changes only
moderately, indicating that Goldilocks does not rely on a narrowly tuned
candidate-pool size. The corresponding learning curves are shown in
Figure~\ref{fig:ablation_candidate}.

\begin{figure}[h]
    \centering

    \begin{subfigure}[t]{0.48\linewidth}
        \centering
        \includegraphics[width=\linewidth]{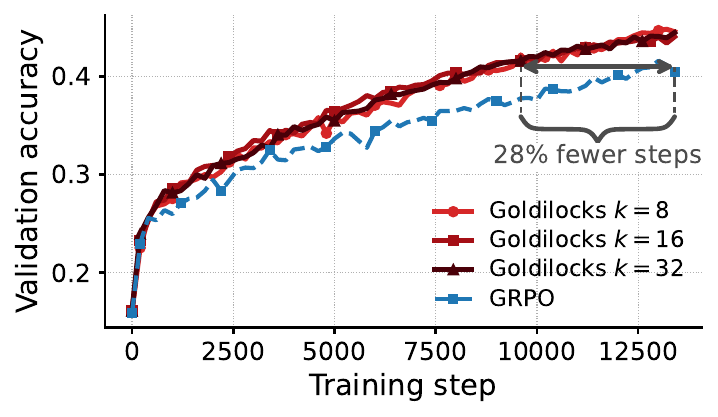}
        \caption{Validation accuracy.}
        \label{fig:ablation_K_eval}
    \end{subfigure}

    \par\medskip

    \begin{subfigure}[t]{0.48\linewidth}
        \centering
        \includegraphics[width=\linewidth]{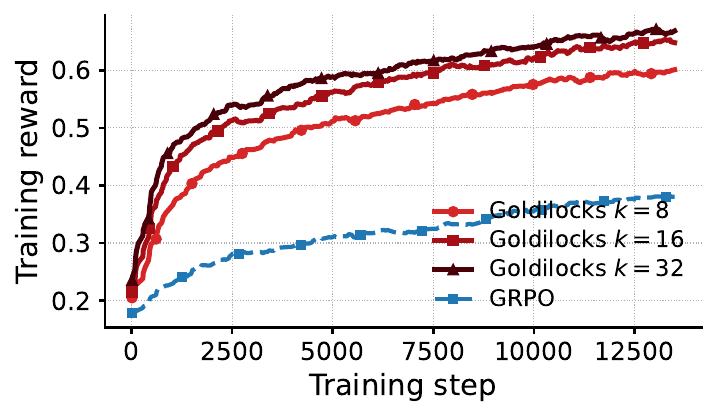}
        \caption{Average training reward.}
        \label{fig:ablation_K_reward}
    \end{subfigure}
    \hfill
    \begin{subfigure}[t]{0.48\linewidth}
        \centering
        \includegraphics[width=\linewidth]{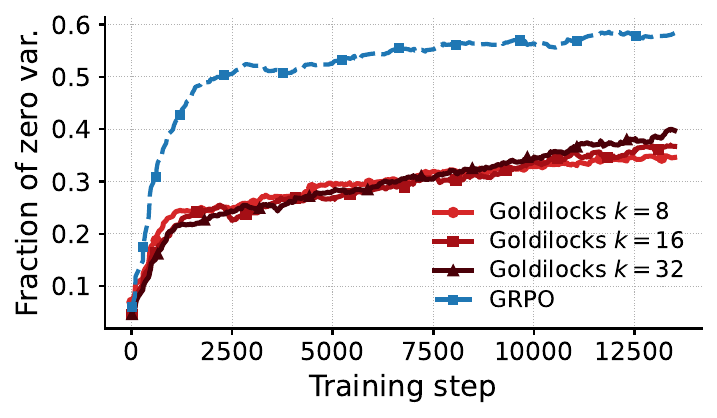}
        \caption{Zero-variance fraction.}
        \label{fig:ablation_K_zero_frac}
    \end{subfigure}

    \caption{\textbf{Candidate-pool size ablation.}
    Training dynamics for $K\in\{8, 16, 32\}$.}
    \label{fig:ablation_candidate}
\end{figure}

\clearpage

\paragraph{Generations per question.}
We compare GRPO and Goldilocks at $G\in\{8,16,32\}$ while keeping the global
batch size fixed. Increasing $G$ therefore allocates more generations to each
question while reducing the number of distinct prompt groups in an update.
It also changes both the noise in the Selector target and the probability of
observing mixed rewards within a rollout group.
Figure~\ref{fig:ablation_generations} shows that Goldilocks retains its
advantage across the tested group sizes.

\begin{figure}[htbp]
    \centering

    \begin{subfigure}[t]{0.68\linewidth}
        \centering
        \includegraphics[width=\linewidth]
        {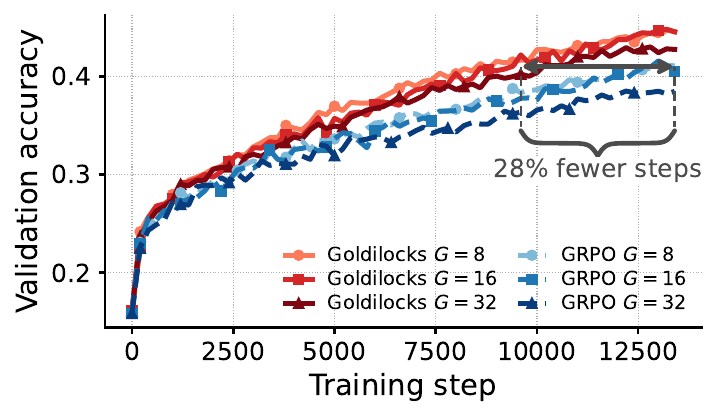}
        \caption{Validation accuracy.}
        \label{fig:ablation_generations_eval}
    \end{subfigure}

    \par\medskip

    \begin{subfigure}[t]{0.48\linewidth}
        \centering
        \includegraphics[width=\linewidth]
        {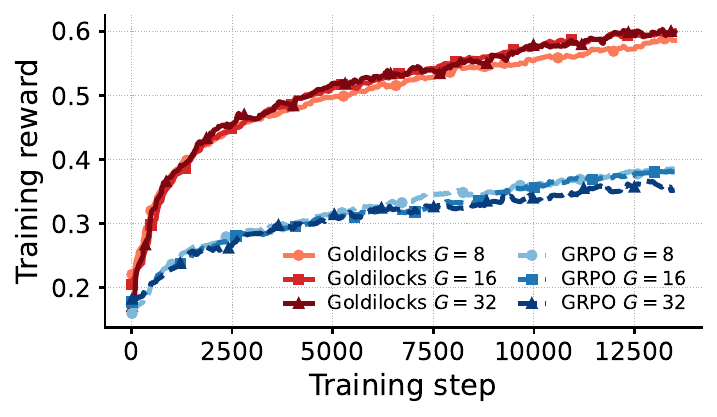}
        \caption{Average training reward.}
        \label{fig:ablation_generations_reward}
    \end{subfigure}
    \hfill
    \begin{subfigure}[t]{0.48\linewidth}
        \centering
        \includegraphics[width=\linewidth]
        {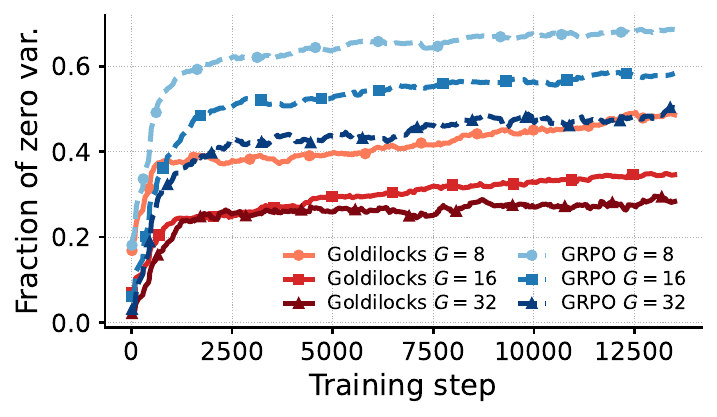}
        \caption{Zero-variance fraction.}
        \label{fig:ablation_generations_zero_frac}
    \end{subfigure}

    \caption{\textbf{Rollout-group-size ablation.}
    GRPO and Goldilocks with $G\in\{8,16,32\}$ at a fixed global batch size.}
    \label{fig:ablation_generations}
\end{figure}

\clearpage
\paragraph{Robustness across seeds.}
Finally, we repeat the standard GRPO--Goldilocks comparison across four Model seeds. Figure~\ref{fig:ablation_seeds} shows the mean learning trajectories and their
variation across seeds.

\begin{figure}[htbp]
    \centering
    \includegraphics[width=0.68\linewidth]
    {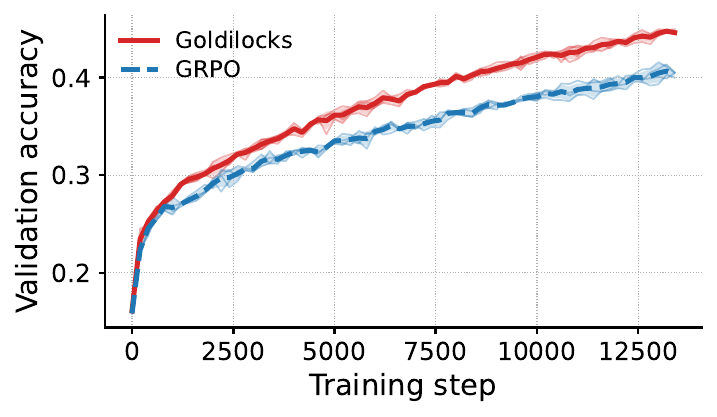}
    \caption{\textbf{Seed robustness.}
    Validation accuracy for GRPO and Goldilocks across four Model seeds.}
    \label{fig:ablation_seeds}
\end{figure}

\clearpage

\clearpage
\section{Additional Validation and Training Dynamics}
\label{app:additional_curves}

We provide additional validation and training-dynamics curves for all Models on OpenMathReasoning and Polaris. For each model, we report validation
accuracy throughout training, the average training reward of the selected
questions, and the fraction of rollout groups with zero reward variance.
Validation accuracy characterizes learning progress, while the latter two
metrics characterize the distribution and informativeness of the questions
selected by each sampling strategy. In particular, the zero-variance fraction
directly measures how often all generations for a selected question receive
the same reward and therefore provide no relative reward signal.


\subsection{OpenMathReasoning}
\label{app:openmath_curves}

Figures~\ref{fig:openmath_olmo}--\ref{fig:openmath_qwen7b} report the
corresponding curves for all Models on OpenMathReasoning. These
results complement Table~\ref{tab:openmath_results} by showing how validation
performance and the training distribution evolve over the full optimization
horizon.


\begin{figure}[h]
    \centering

    \begin{subfigure}[t]{0.48\linewidth}
        \centering
        \includegraphics[width=\linewidth]{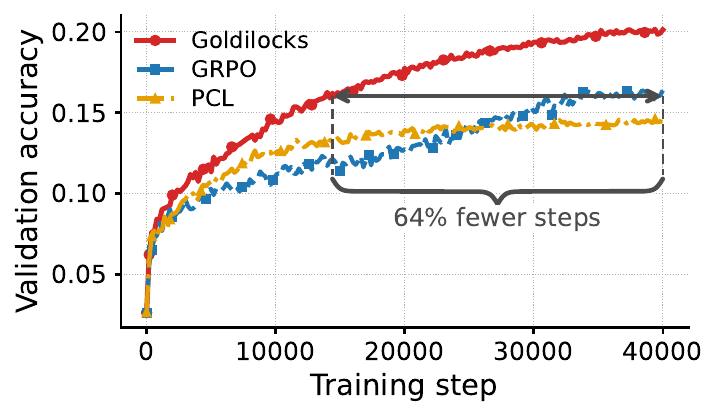}
        \caption{Validation accuracy.}
    \end{subfigure}
    \hfill
    \begin{subfigure}[t]{0.48\linewidth}
        \centering
        \includegraphics[width=\linewidth]{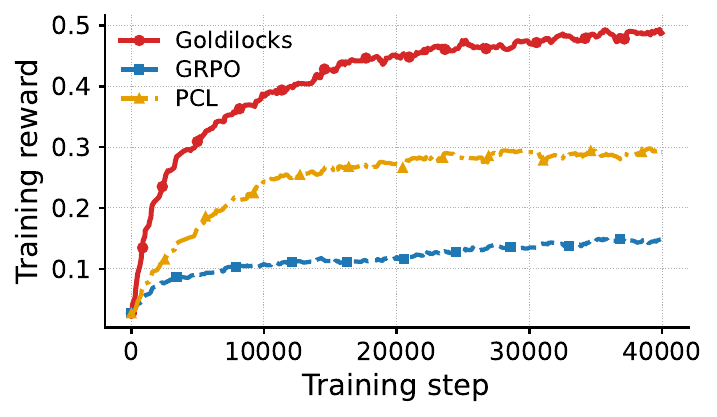}
        \caption{Average training reward.}
    \end{subfigure}

    \par\medskip

    \begin{subfigure}[t]{0.48\linewidth}
        \centering
        \includegraphics[width=\linewidth]{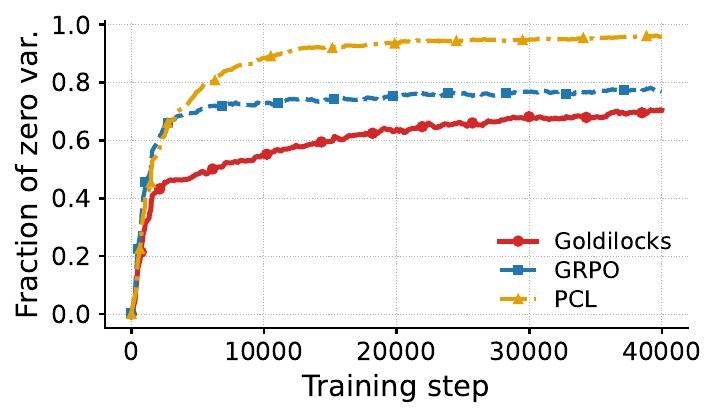}
        \caption{Zero-variance fraction.}
    \end{subfigure}

    \caption{\textbf{OLMo2-1B on OpenMathReasoning.}
    Validation performance and training dynamics throughout optimization.}
    \label{fig:openmath_olmo}
\end{figure}


\begin{figure}[h]
    \centering

    \begin{subfigure}[t]{0.48\linewidth}
        \centering
        \includegraphics[width=\linewidth]{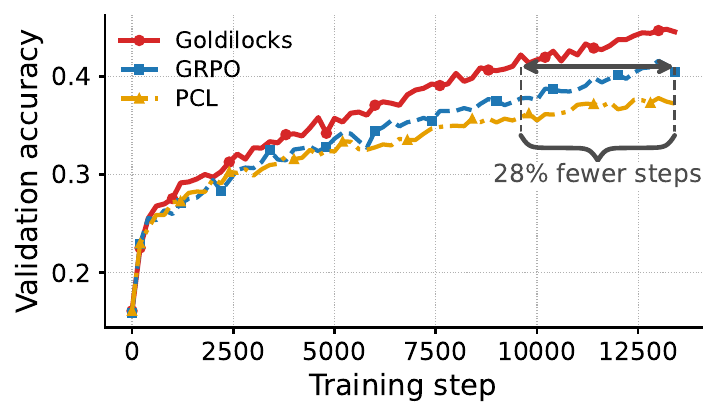}
        \caption{Validation accuracy.}
    \end{subfigure}
    \hfill
    \begin{subfigure}[t]{0.48\linewidth}
        \centering
        \includegraphics[width=\linewidth]{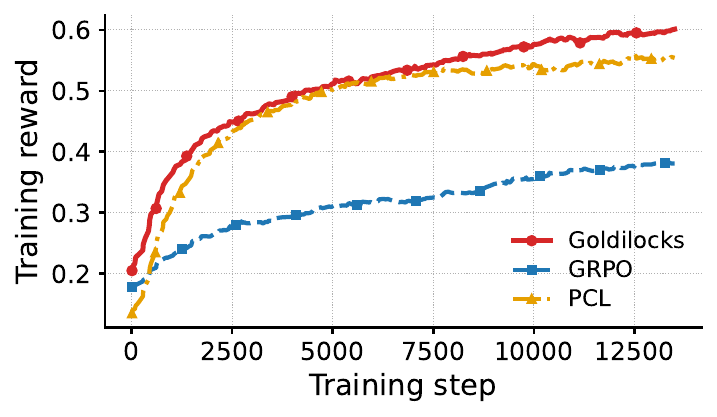}
        \caption{Average training reward.}
    \end{subfigure}

    \par\medskip

    \begin{subfigure}[t]{0.48\linewidth}
        \centering
        \includegraphics[width=\linewidth]{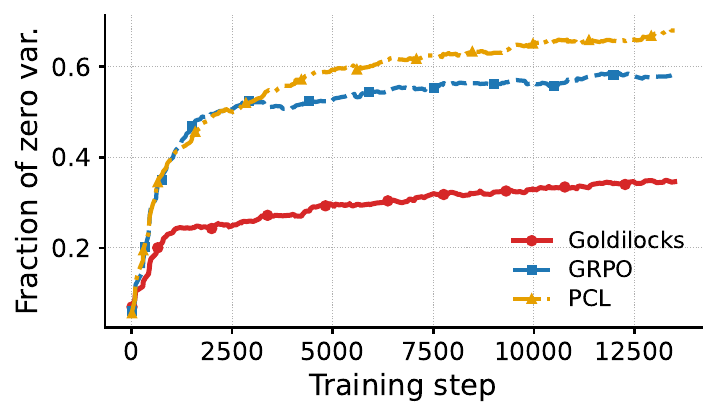}
        \caption{Zero-variance fraction.}
    \end{subfigure}

    \caption{\textbf{Qwen2.5-1.5B on OpenMathReasoning.}
    Validation performance and training dynamics throughout optimization.}
    \label{fig:openmath_qwen15}
\end{figure}


\begin{figure}[h]
    \centering

    \begin{subfigure}[t]{0.48\linewidth}
        \centering
        \includegraphics[width=\linewidth]{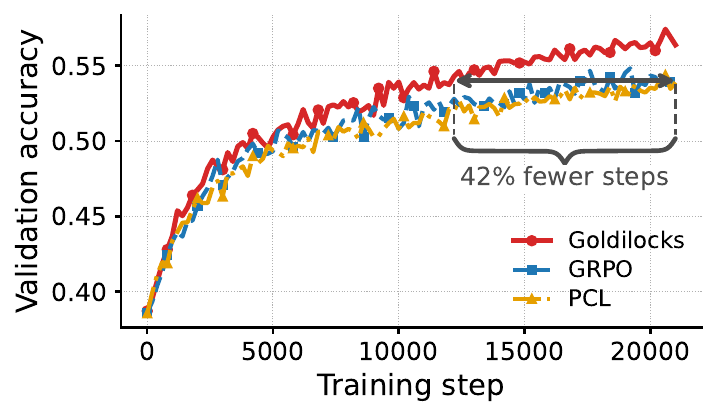}
        \caption{Validation accuracy.}
    \end{subfigure}
    \hfill
    \begin{subfigure}[t]{0.48\linewidth}
        \centering
        \includegraphics[width=\linewidth]{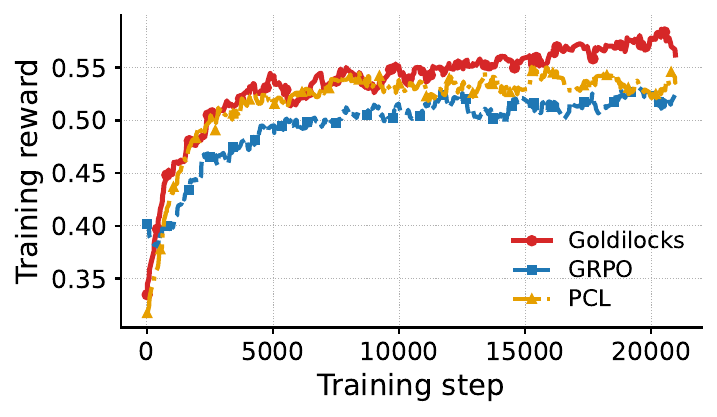}
        \caption{Average training reward.}
    \end{subfigure}

    \par\medskip

    \begin{subfigure}[t]{0.48\linewidth}
        \centering
        \includegraphics[width=\linewidth]{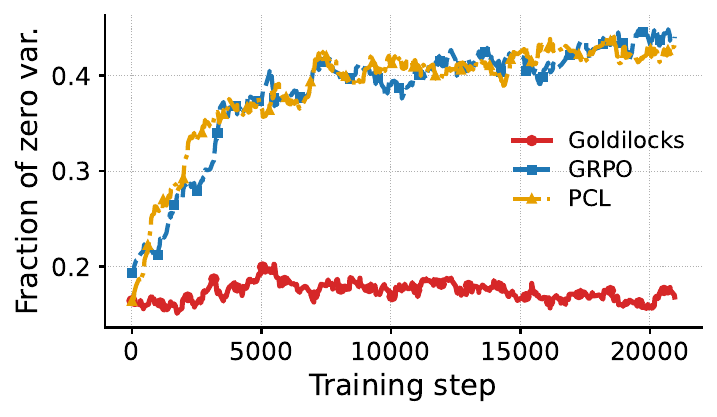}
        \caption{Zero-variance fraction.}
    \end{subfigure}

    \caption{\textbf{Qwen3-4B on OpenMathReasoning.}
    Validation performance and training dynamics throughout optimization.}
    \label{fig:openmath_qwen3}
\end{figure}


\begin{figure}[h]
    \centering

    \begin{subfigure}[t]{0.48\linewidth}
        \centering
        \includegraphics[width=\linewidth]{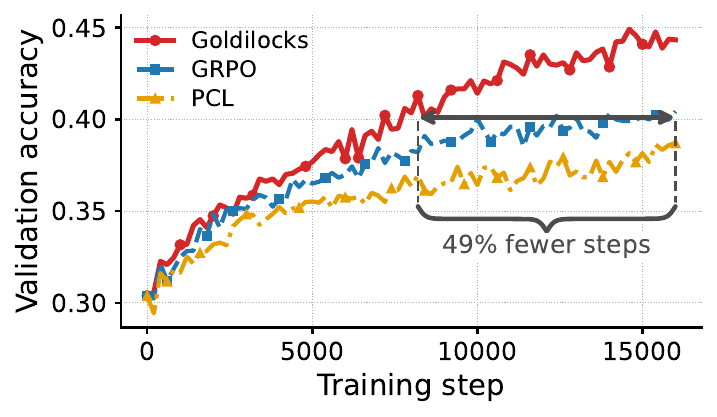}
        \caption{Validation accuracy.}
    \end{subfigure}
    \hfill
    \begin{subfigure}[t]{0.48\linewidth}
        \centering
        \includegraphics[width=\linewidth]{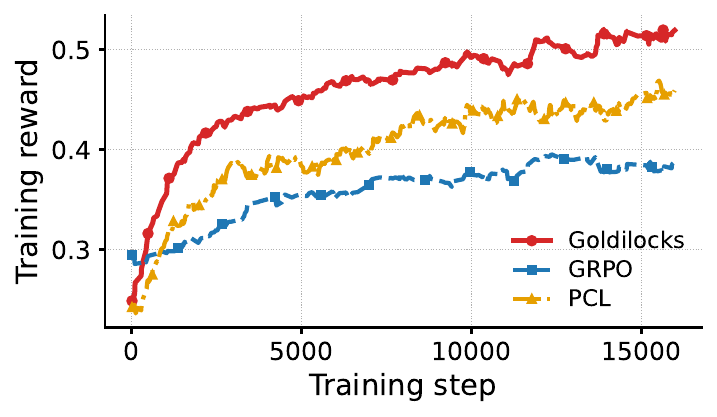}
        \caption{Average training reward.}
    \end{subfigure}

    \par\medskip

    \begin{subfigure}[t]{0.48\linewidth}
        \centering
        \includegraphics[width=\linewidth]{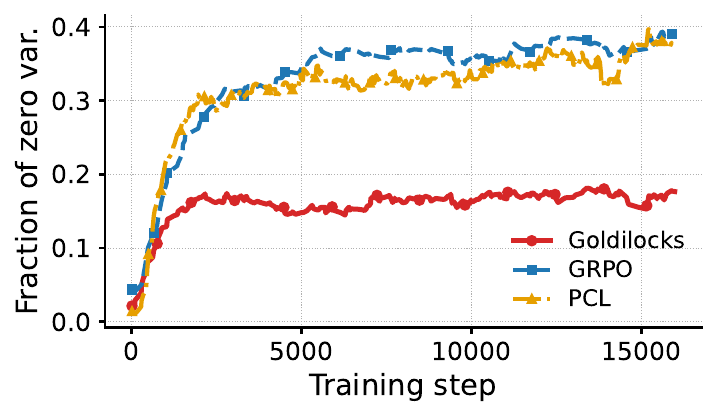}
        \caption{Zero-variance fraction.}
    \end{subfigure}

    \caption{\textbf{Phi-4-mini-Instruct on OpenMathReasoning.}
    Validation performance and training dynamics throughout optimization.}
    \label{fig:openmath_phi}
\end{figure}


\begin{figure}[h]
    \centering

    \begin{subfigure}[t]{0.48\linewidth}
        \centering
        \includegraphics[width=\linewidth]{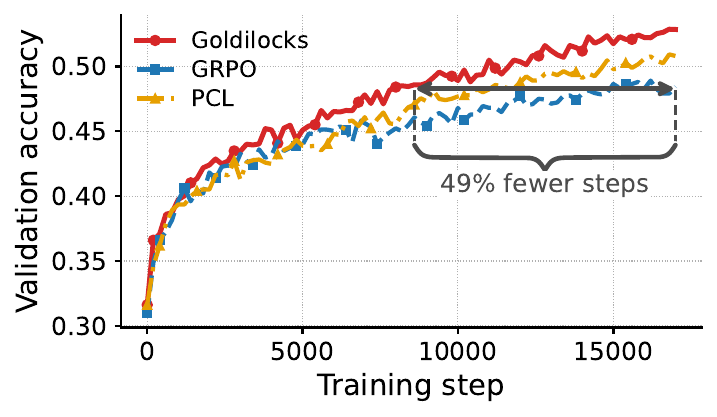}
        \caption{Validation accuracy.}
    \end{subfigure}
    \hfill
    \begin{subfigure}[t]{0.48\linewidth}
        \centering
        \includegraphics[width=\linewidth]{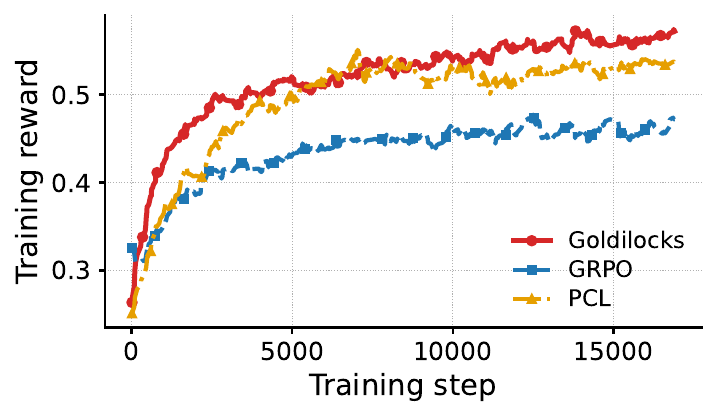}
        \caption{Average training reward.}
    \end{subfigure}

    \par\medskip

    \begin{subfigure}[t]{0.48\linewidth}
        \centering
        \includegraphics[width=\linewidth]{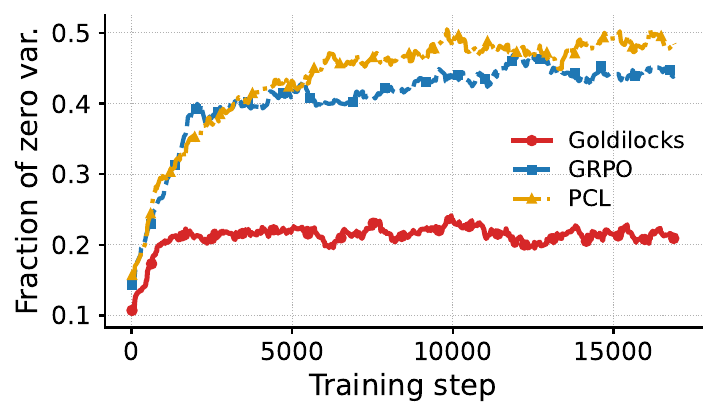}
        \caption{Zero-variance fraction.}
    \end{subfigure}

    \caption{\textbf{Qwen2.5-7B on OpenMathReasoning.}
    Validation performance and training dynamics throughout optimization.}
    \label{fig:openmath_qwen7b}
\end{figure}

\clearpage


\subsection{Polaris}
\label{app:polaris_curves}

Figures~\ref{fig:polaris_olmo}--\ref{fig:polaris_qwen7b} report the same
diagnostics on Polaris. The smaller training pool produces more frequent
question revisitation, allowing history-based methods such as MoPPS to
accumulate increasingly informative question-specific estimates. These curves
complement Table~\ref{tab:polaris_results} by showing how the relative
performance and selected-question distributions evolve throughout training.


\begin{figure}[h]
    \centering

    \begin{subfigure}[t]{0.48\linewidth}
        \centering
        \includegraphics[width=\linewidth]{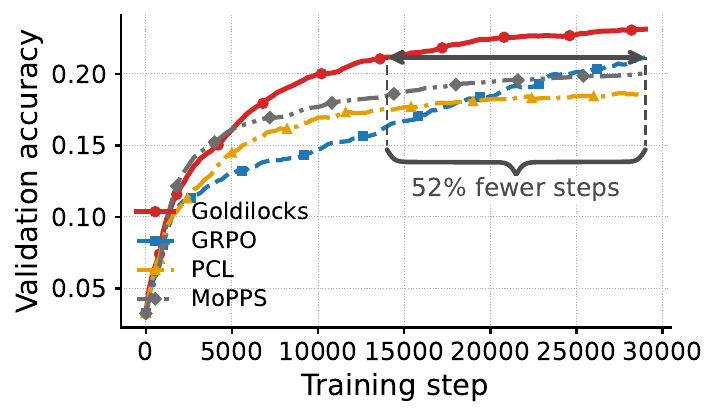}
        \caption{Validation accuracy.}
    \end{subfigure}
    \hfill
    \begin{subfigure}[t]{0.48\linewidth}
        \centering
        \includegraphics[width=\linewidth]{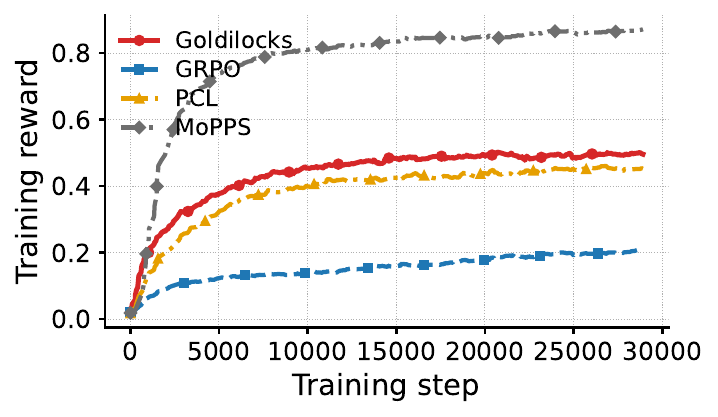}
        \caption{Average training reward.}
    \end{subfigure}

    \par\medskip

    \begin{subfigure}[t]{0.48\linewidth}
        \centering
        \includegraphics[width=\linewidth]{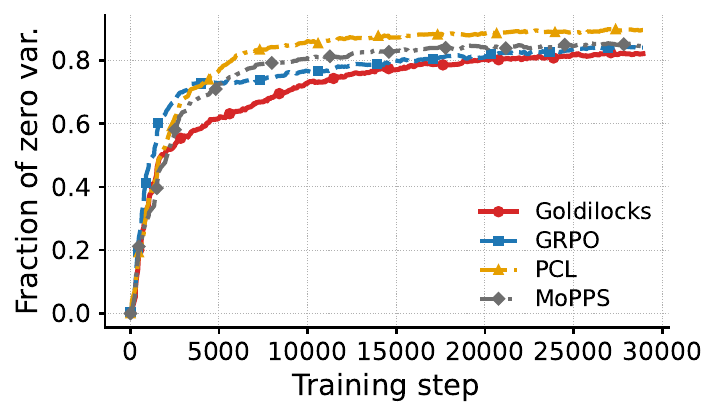}
        \caption{Zero-variance fraction.}
    \end{subfigure}

    \caption{\textbf{OLMo2-1B on Polaris.}
    Validation performance and training dynamics throughout optimization.}
    \label{fig:polaris_olmo}
\end{figure}


\begin{figure}[h]
    \centering

    \begin{subfigure}[t]{0.48\linewidth}
        \centering
        \includegraphics[width=\linewidth]{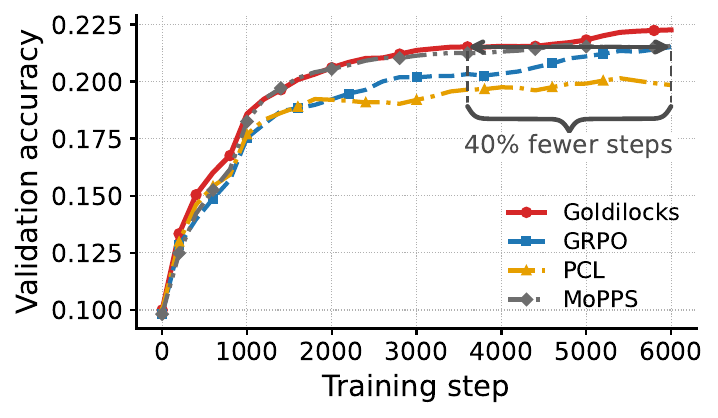}
        \caption{Validation accuracy.}
    \end{subfigure}
    \hfill
    \begin{subfigure}[t]{0.48\linewidth}
        \centering
        \includegraphics[width=\linewidth]{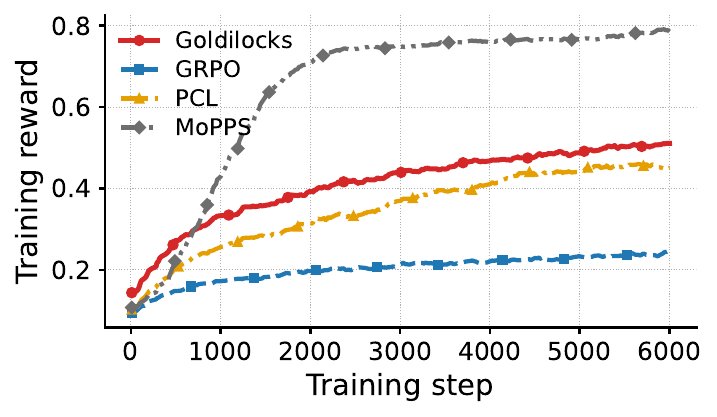}
        \caption{Average training reward.}
    \end{subfigure}

    \par\medskip

    \begin{subfigure}[t]{0.48\linewidth}
        \centering
        \includegraphics[width=\linewidth]{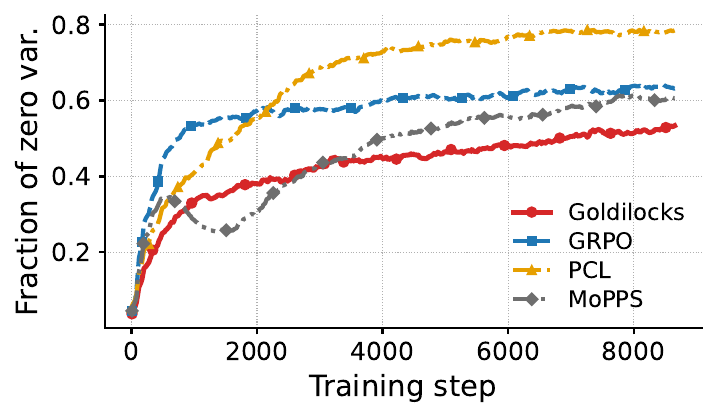}
        \caption{Zero-variance fraction.}
    \end{subfigure}

    \caption{\textbf{Qwen2.5-1.5B on Polaris.}
    Validation performance and training dynamics throughout optimization.}
    \label{fig:polaris_qwen15}
\end{figure}


\begin{figure}[h]
    \centering

    \begin{subfigure}[t]{0.48\linewidth}
        \centering
        \includegraphics[width=\linewidth]{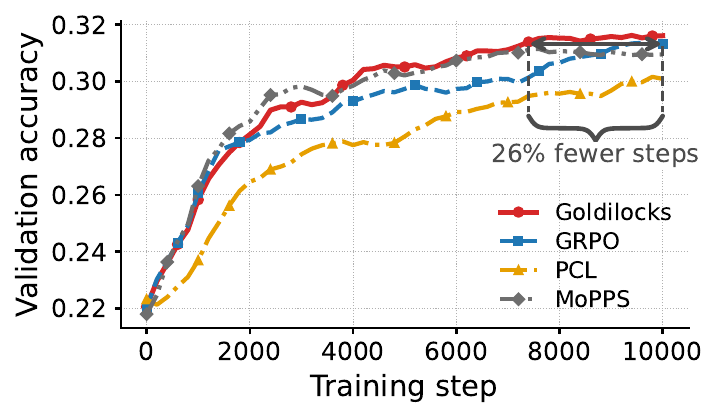}
        \caption{Validation accuracy.}
    \end{subfigure}
    \hfill
    \begin{subfigure}[t]{0.48\linewidth}
        \centering
        \includegraphics[width=\linewidth]{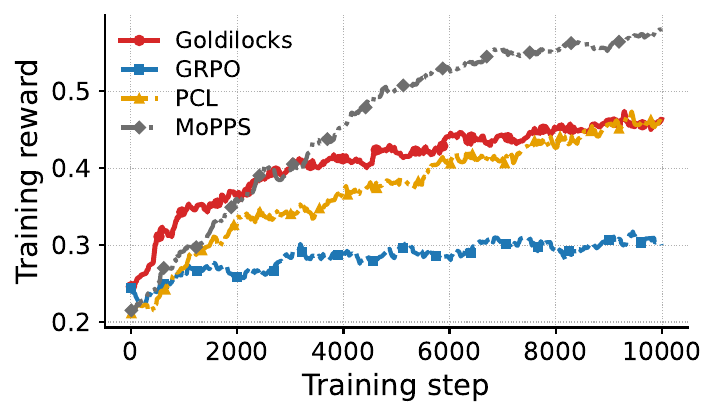}
        \caption{Average training reward.}
    \end{subfigure}

    \par\medskip

    \begin{subfigure}[t]{0.48\linewidth}
        \centering
        \includegraphics[width=\linewidth]{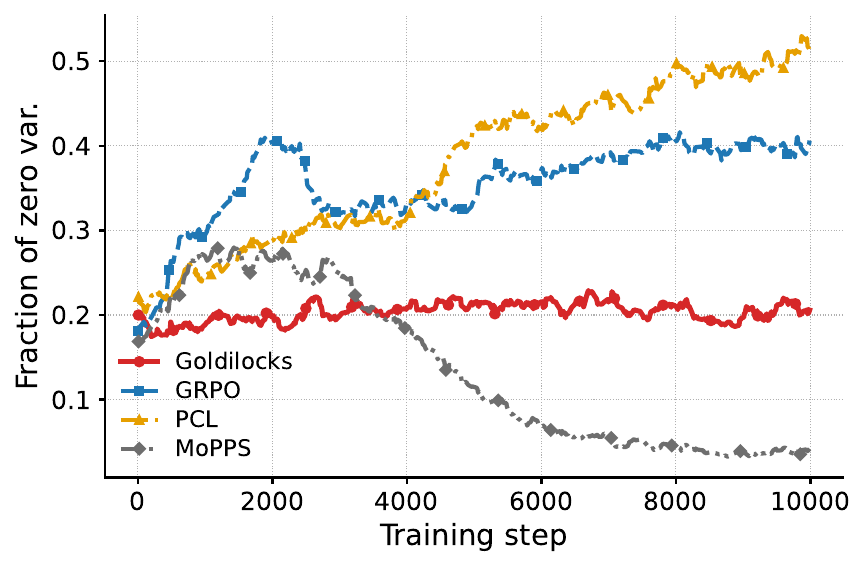}
        \caption{Zero-variance fraction.}
    \end{subfigure}

    \caption{\textbf{Qwen3-4B on Polaris.}
    Validation performance and training dynamics throughout optimization.}
    \label{fig:polaris_qwen3}
\end{figure}


\begin{figure}[h]
    \centering

    \begin{subfigure}[t]{0.48\linewidth}
        \centering
        \includegraphics[width=\linewidth]{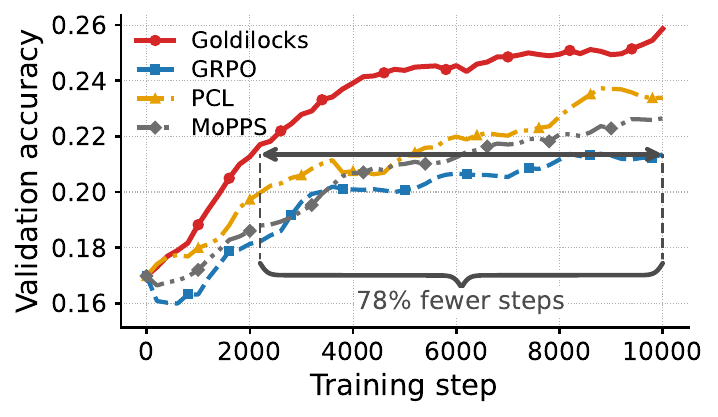}
        \caption{Validation accuracy.}
    \end{subfigure}
    \hfill
    \begin{subfigure}[t]{0.48\linewidth}
        \centering
        \includegraphics[width=\linewidth]{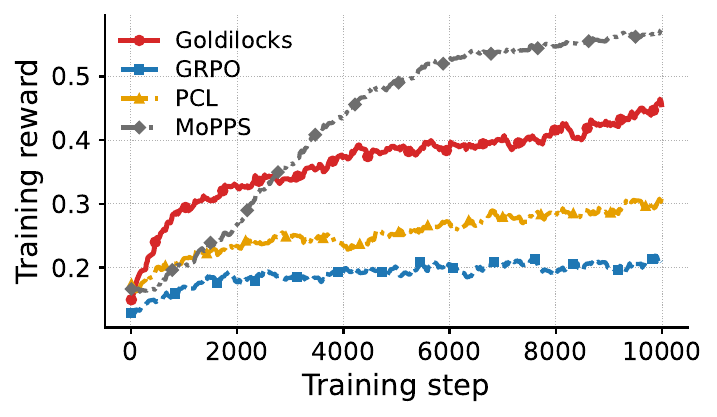}
        \caption{Average training reward.}
    \end{subfigure}

    \par\medskip

    \begin{subfigure}[t]{0.48\linewidth}
        \centering
        \includegraphics[width=\linewidth]{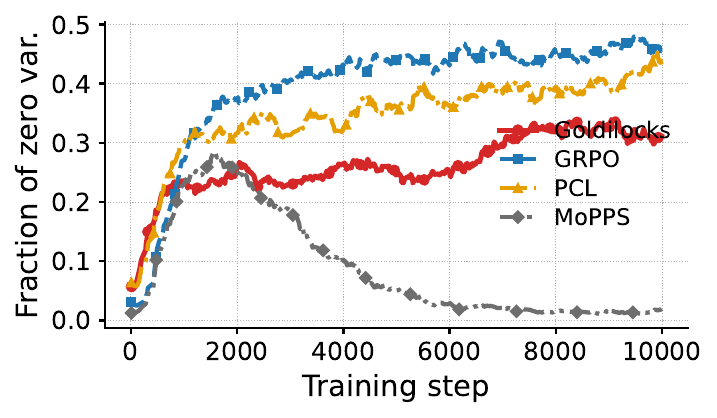}
        \caption{Zero-variance fraction.}
    \end{subfigure}

    \caption{\textbf{Phi-4-mini-Instruct on Polaris.}
    Validation performance and training dynamics throughout optimization.}
    \label{fig:polaris_phi}
\end{figure}


\begin{figure}[h]
    \centering

    \begin{subfigure}[t]{0.48\linewidth}
        \centering
        \includegraphics[width=\linewidth]{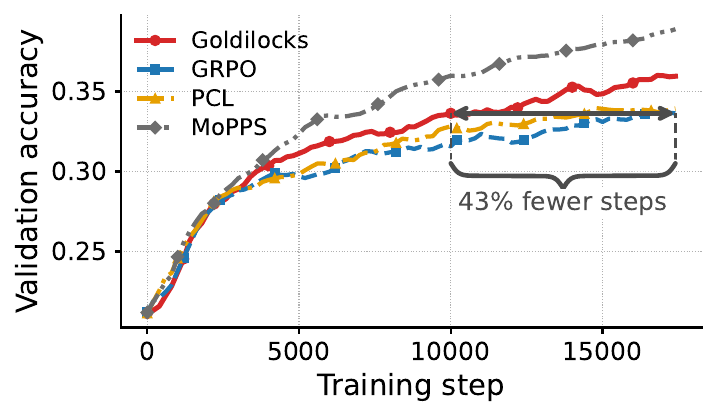}
        \caption{Validation accuracy.}
    \end{subfigure}
    \hfill
    \begin{subfigure}[t]{0.48\linewidth}
        \centering
        \includegraphics[width=\linewidth]{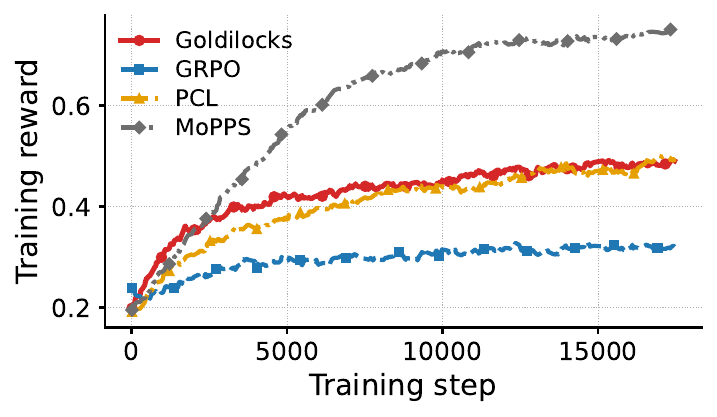}
        \caption{Average training reward.}
    \end{subfigure}

    \par\medskip

    \begin{subfigure}[t]{0.48\linewidth}
        \centering
        \includegraphics[width=\linewidth]{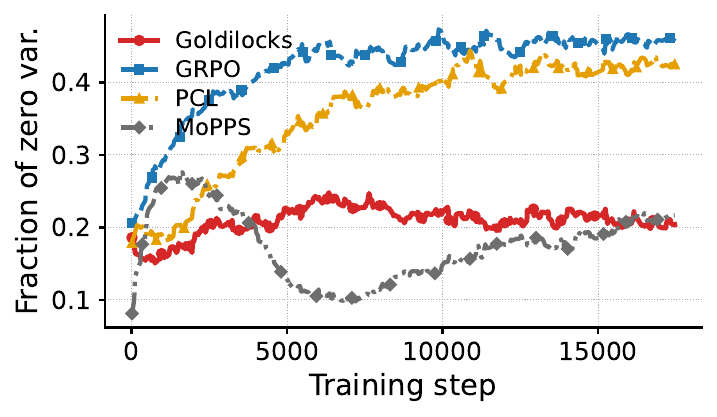}
        \caption{Zero-variance fraction.}
    \end{subfigure}

    \caption{\textbf{Qwen2.5-7B on Polaris.}
    Validation performance and training dynamics throughout optimization.}
    \label{fig:polaris_qwen7b}
\end{figure}

\clearpage
\clearpage

\section{Prompts}\label{app:prompts}
The final input to the model is constructed by concatenating a fixed header, the specific math problem, and a fixed footer:

\[
\text{Input} = \texttt{PROMPT\_HEADER} + \text{Question} + \texttt{PROMPT\_FOOTER}
\]

The exact text used for the header and footer is provided below.

\subsection*{Prompt Header}
\begin{verbatim}
You are an expert problem solver. Your task is to solve the problem by 
explicitly showing your detailed Chain of Thought.
You must write out your reasoning step-by-step, explaining your logic 
for every calculation or deduction.
**Crucially, the final answer must be enclosed in a \boxed{...} command.**
\end{verbatim}

\subsection*{Prompt Footer}
\begin{verbatim}
Strictly follow this output format:
### Step-by-Step Reasoning
<Write your detailed thought process here. Do not skip steps.>

### Final Answer
\boxed{<your final answer>}
**Important: The \boxed{...} command should only contain the final, 
simplified numerical or symbolic answer, with no extra text or units.**
\end{verbatim}

\clearpage

\section{Success Probability versus Reward-Variability Targets}
\label{app:targets}
\label{app:selection_target_intuition}

If the success probability of every question were known exactly, a curriculum
targeting $p_q \approx 1/2$ and one maximizing the ideal reward-variability
score $\sqrt{p_q(1-p_q)}$ would favor the same intermediate-difficulty
questions. The distinction becomes important in practice because $p_q$ is not
known and must instead be inferred from the question and from finite rollout
feedback.

A predicted success probability near $1/2$ does not necessarily mean that the
underlying question itself has success probability near $1/2$. For example, the
same prediction can arise when the information available to the predictor
cannot distinguish between questions that are very easy and questions that are
very difficult. Goldilocks does not train the Selector only to predict mean
success probability; instead, it directly predicts the reward variability
realized by rollout groups.

\subsection{Setup}

We analyze a fixed Model checkpoint with parameters $\theta$; all quantities
below should therefore be understood as checkpoint-specific. Let
$Q\sim\mathcal{D}$ denote a question drawn from the candidate-question
distribution, and let
\[
H = h(Q)
\]
denote the representation of that question available to the Selector.

For a particular question $q$, define
\[
p_q(\theta)
=
\Pr_{o\sim\pi_\theta(\cdot\mid q)}
\bigl(r_{\mathrm{ver}}(o)=1\bigr),
\]
the probability that the current Model produces a correct answer under the
binary verification reward. We then define the random variable
\[
P = p_Q(\theta),
\]
which varies because the randomly drawn question $Q$ varies.

Conditional on the information available to the Selector, we define
\begin{equation}
    \mu(h)
    \coloneqq
    \mathbb{E}[P\mid H=h],
    \qquad
    \tau^2(h)
    \coloneqq
    \mathrm{Var}(P\mid H=h).
    \label{eq:targets-moments}
\end{equation}
Here, $\mu(h)$ is the mean solvability of questions represented by $h$, while
$\tau^2(h)$ measures how much their true success probabilities can differ even
after conditioning on the representation available to the Selector. Thus,
$\tau^2(h)$ captures representation-conditional ambiguity rather than
finite-sample estimation error.

For a rollout group of size $G$, let
$R_1,\ldots,R_G\in\{0,1\}$ denote the verification rewards of independently
sampled Model outputs. Conditional on a question with success probability $P$,
\[
R_i\mid P \sim \mathrm{Bernoulli}(P).
\]
Let
\[
S=\sum_{i=1}^{G}R_i,
\qquad
S\mid P\sim\mathrm{Binomial}(G,P),
\qquad
\hat p=\frac{S}{G}.
\]
We use
\[
\sigma(p)=\sqrt{p(1-p)}
\]
for the standard deviation of a Bernoulli variable with success probability
$p$, and define the empirical Goldilocks target as
\[
\hat\sigma
=
\sigma(\hat p)
=
\sqrt{\hat p(1-\hat p)}.
\]

The Selector is trained to predict $\hat\sigma$. Notice that
$\hat\sigma=0$ exactly when $S=0$ or $S=G$, i.e., when every rollout in the
group receives the same verification reward. Such groups contain no
within-group reward variation and therefore provide no relative reward signal
to GRPO.

We call a rollout group \emph{informative} when it contains both correct and
incorrect responses, i.e., when $0<S<G$. Define
\begin{equation}
\rho(h)
\coloneqq
\Pr(0<S<G\mid H=h).
\label{eq:targets-rho-def}
\end{equation}
Conditioning first on $P$, the probability of a zero-variance group is
$P^G+(1-P)^G$. Averaging over questions represented by $h$ therefore gives
\begin{equation}
\rho(h)
=
1-
\mathbb{E}\!\left[
P^G+(1-P)^G
\mid H=h
\right].
\label{eq:targets-rho}
\end{equation}
Thus, $\rho(h)$ averages over both variation among questions consistent with
the representation $h$ and the randomness of Model rollouts.

\subsection{What the two prediction targets learn}

We first consider a predictor trained with squared loss on the empirical success
rate $\hat p$. The population minimizer of squared-error regression is the
conditional expectation of its target. Since
\[
\mathbb{E}[\hat p\mid P]=P,
\]
the tower property gives
\begin{equation}
m(h)
=
\mathbb{E}[\hat p\mid H=h]
=
\mathbb{E}[P\mid H=h]
=
\mu(h).
\label{eq:targets-mean-predictor}
\end{equation}
Thus, in the population limit, a success-probability predictor such as that used
by PCL estimates the conditional mean solvability $\mu(h)$. A curriculum
targeting success probability $1/2$ therefore favors representations for which
$\mu(h)$ is close to $1/2$.

Goldilocks instead trains the Selector on $\hat\sigma$. Under the same
population squared-loss argument, its prediction target is
\begin{equation}
v(h)
=
\mathbb{E}[\hat\sigma\mid H=h].
\label{eq:targets-variability-predictor}
\end{equation}
Hence, $v(h)$ is the expected empirical reward variability of the rollout group
that would be generated after selecting a question represented by $h$.

\subsection{Why the predicted mean can be misleading}

The conditional mean $\mu(h)$ does not fully characterize how informative the
underlying questions are. To see this, consider the expected Bernoulli reward
variance:
\begin{align}
\mathbb{E}[P(1-P)\mid H=h]
&=
\mathbb{E}[P\mid H=h]
-
\mathbb{E}[P^2\mid H=h] \\
&=
\mu(h)
-
\left(\tau^2(h)+\mu(h)^2\right) \\
&=
\mu(h)(1-\mu(h))-\tau^2(h).
\label{eq:targets-decomposition}
\end{align}
Therefore, the variance obtained by simply applying the Bernoulli formula to
the predicted mean,
$\mu(h)(1-\mu(h))$, is larger than the expected within-question reward
variance by exactly $\tau^2(h)$.

A related consequence follows for the reward standard deviation. Because the
square-root function is concave, Jensen's inequality gives
\begin{equation}
\mathbb{E}[\sigma(P)\mid H=h]
\leq
\sqrt{
\mathbb{E}[P(1-P)\mid H=h]
}
=
\sqrt{
\mu(h)(1-\mu(h))-\tau^2(h)
}
\leq
\sigma(\mu(h)).
\label{eq:targets-jensen-sigma}
\end{equation}
Thus, computing variability from the conditional mean can systematically
overstate the reward variability of the underlying questions whenever
$\tau^2(h)>0$.

The same issue appears directly in the probability of obtaining a
zero-variance rollout group. For $G\geq2$, the function
\[
z(p)=p^G+(1-p)^G
\]
is convex on $[0,1]$. Jensen's inequality therefore yields
\begin{equation}
\Pr(\text{zero variance}\mid H=h)
=
\mathbb{E}[z(P)\mid H=h]
\geq
z(\mu(h))
=
\mu(h)^G+(1-\mu(h))^G.
\label{eq:targets-jensen-zv}
\end{equation}
A mean-based estimate therefore gives a lower bound on the actual
zero-variance probability conditional on $h$.

Consequently, a predicted mean near $1/2$ is not sufficient to guarantee that
the selected questions are themselves near the Model's solvability frontier.
As an extreme illustration, suppose that conditional on some representation
$h$, half of the questions have $P=0$ and half have $P=1$. Then
\[
\mu(h)=\frac12,
\]
so a mean-based curriculum targeting $1/2$ gives this representation its
highest possible score. Nevertheless, every individual question has
deterministic outcomes, and every rollout group is zero-variance.

This also explains why an average training reward near $1/2$ is not, by
itself, evidence that the selected questions individually have success
probability near $1/2$. The same aggregate mean can arise from a mixture of
questions near the two extremes.

\subsection{The Goldilocks score and informative rollout groups}

The Goldilocks target has a direct relationship to the probability of obtaining
an informative rollout group. Define
\[
c_G
\coloneqq
\sigma\!\left(\frac1G\right)
=
\frac{\sqrt{G-1}}{G}.
\]
For a zero-variance group, $\hat\sigma=0$. For an informative group,
$S\in\{1,\ldots,G-1\}$, so
\[
\hat p\in
\left[\frac1G,1-\frac1G\right].
\]
Since $\sigma(p)=\sqrt{p(1-p)}$ is symmetric around $1/2$, is increasing on
$[0,1/2]$, and has maximum $1/2$, every informative group satisfies
\[
c_G
\leq
\hat\sigma
\leq
\frac12.
\]
Equivalently,
\[
c_G\,\mathbf{1}\{0<S<G\}
\leq
\hat\sigma
\leq
\frac12\,\mathbf{1}\{0<S<G\}.
\]
Taking conditional expectations given $H=h$ gives
\begin{equation}
c_G\,\rho(h)
\leq
v(h)
\leq
\frac12\,\rho(h).
\label{eq:targets-sandwich}
\end{equation}

Rearranging,
\begin{equation}
2v(h)
\leq
\rho(h)
\leq
\frac{v(h)}{c_G}.
\label{eq:targets-rho-bounds}
\end{equation}
Thus, the quantity predicted by the Goldilocks Selector is directly tied to the
probability that a selected question produces a mixed-reward rollout group.
In particular, a high value of $v(h)$ implies a nontrivial lower bound on the
probability of obtaining an informative group. The conditional mean
$\mu(h)$ alone does not provide such a guarantee.

The connection to representation-conditional heterogeneity is especially
transparent if we examine the \emph{squared} empirical variability target.
Since
\[
\mathrm{Var}(\hat p\mid P)
=
\frac{P(1-P)}{G}
\]
and $\mathbb{E}[\hat p\mid P]=P$, we have
\begin{align}
\mathbb{E}[\hat p(1-\hat p)\mid P]
&=
\mathbb{E}[\hat p\mid P]
-
\mathbb{E}[\hat p^2\mid P] \\
&=
\left(1-\frac1G\right)P(1-P).
\end{align}
Taking the conditional expectation over $P$ given $H=h$ and applying
Eq.~\ref{eq:targets-decomposition} gives
\begin{equation}
\mathbb{E}\!\left[
\hat p(1-\hat p)\mid H=h
\right]
=
\left(1-\frac1G\right)
\left[
\mu(h)(1-\mu(h))-\tau^2(h)
\right].
\label{eq:targets-uncertainty}
\end{equation}
Thus, for the squared variability target, the contribution suggested by the
mean success probability is reduced exactly by the conditional heterogeneity
$\tau^2(h)$, up to the finite-group factor $1-1/G$.

Goldilocks uses the standard-deviation target
$\hat\sigma=\sqrt{\hat p(1-\hat p)}$ rather than its square, so
Eq.~\ref{eq:targets-uncertainty} is not itself the exact population target
learned by Goldilocks. The direct relationship for the actual Goldilocks target
is instead given by the informative-group bounds in
Eq.~\ref{eq:targets-sandwich}.

Overall, this analysis predicts that a success-probability curriculum can
maintain an average selected reward near $1/2$ while still selecting a
heterogeneous mixture of very easy and very difficult questions. Goldilocks
instead trains the Selector directly on realized within-group reward
variability, which is tied to the probability of obtaining mixed-reward
groups. We therefore expect Goldilocks to produce fewer zero-variance rollout
groups even when another curriculum maintains an average reward closer to
$1/2$. This prediction is consistent with the training dynamics in
Figure~\ref{fig:qwen7b_training_dynamics}.
\clearpage
\section{Computational Cost of Experiments}
\label{app:comp-cost}
All experiments were conducted on 8-GPU nodes, using either H100 80GB or H200 141GB GPUs. For GRPO and MoPPS, all eight GPUs are used for model training, whereas Goldilocks and PCL allocate seven GPUs to the model and one GPU to the Selector. Across the experiments reported in this work, we ran 66 jobs totaling approximately 115{,}099 GPU-hours. This includes 54{,}891 GPU-hours for the main experiments, 12{,}103 for seed repetitions, 8{,}133 for objective ablations, 4{,}159 for the exploration-probability ablation, 8{,}318 for the rollout-group-size ablation, 4{,}159 for the candidate-pool-size ablation, and 23{,}336 for batch-size tuning.%

\end{document}